\documentclass[11pt]{article}
\usepackage[utf8]{inputenc}
\usepackage[T1]{fontenc}    
\usepackage{hyperref}       
\usepackage{url}            
\usepackage{booktabs}       
\usepackage{amsfonts}       
\usepackage{nicefrac}       
\usepackage{algorithm,algorithmic,amsmath,amssymb,amsthm}
\usepackage{authblk}
\usepackage{graphicx}

\newtheorem{theorem}{Theorem}[section]
\newtheorem{lemma}[theorem]{Lemma}

\newcommand{\eat}[1]{}

\newenvironment{definition}[1][Definition]{\begin{trivlist}
\item[\hskip \labelsep {\bfseries #1}]}{\end{trivlist}}

\DeclareMathOperator*{\argmax}{argmax}
\DeclareMathOperator*{\argmin}{argmin}

\def \w {{\mathbf{w}}}
\def \x {{\mathbf{x}}}
\def \y {{\mathbf{y}}}
\def \z {{\mathbf{z}}}

\def \bv {\mathbf{v}}
\def \by {\mathbf{y}}

\def \bmu {{\boldsymbol\mu}}
\def \E {{\mathbb{E}}}

\def \X {{\mathcal{X}}}
\def \K {{\mathcal{K}}}

\def \bzeta {\boldsymbol{\zeta}}
\def \bzetaL {\bzeta^{(L)}}
\def \bzetaLt {\bzeta^{(L_{\tau})}}
\def \bzetaLS {\bzeta^{(L \cup S)}}

\def \lbzetaL {l\left(\bzetaL \right)}
\def \lbzetaLS {l\left(\bzetaLS \right)}

\def \gradljbzetaL {\nabla l_{j} \left(\bzetaL \right)}

\def \R {\mathbb{R}}
\def \bonej {\mathbf{1}^{(\{j\})}}

\def \bv {\mathbf{v}}
\def \x {\mathbf{x}}
\def \by {\mathbf{y}}
\def \y {\mathbf{y}}

\def \w {\mathbf{w}}
\def \bzeta {\boldsymbol{\zeta}}
\def \bmu {\boldsymbol{\mu}}
\def \bzetaL {\bzeta^{(L)}}

\def \bzetaLS {\bzeta^{(L \cup S)}}

\def \lbzetaL {l\left(\bzetaL \right)}

\def \lbzetaLS {l\left(\bzetaLS \right)}

\def \lbzero {l\left(\mathbf{0} \right)}
\def \gradljbzetaL {\nabla l_{j} \left(\bzetaL \right)}

\def \gradlSbzetaL {\nabla l_{S} \left(\bzetaL \right)}
\def \gradljpbzetaL {\nabla l_{j}^+\left(\bzetaL \right)}

\def \gradlpbzero {\nabla l_{p} \left( \mathbf{0} \right)}
\def \gradlibzero {\nabla l_{i} \left( \mathbf{0} \right)}
\def \gradljpbzero {\nabla l_{j}^+\left(\mathbf{0} \right)}
\def \gradlppbzero {\nabla l_{p}^+\left(\mathbf{0} \right)}

\def \gradljbzetaLtau {\nabla l_{j}\left(\bzeta^{\left(L_{\tau}\right)}\right)}
\def \gradljpbzetaLtau {\nabla l_{j}^+\left(\bzeta^{\left(L_{\tau}\right)}\right)}
\def \R {\mathbb{R}}
\def \bonej {\mathbf{1}^{(\{j\})}}

\def \byS {\by^{(S)}}

\numberwithin{equation}{section}
\title{Streaming Methods for Restricted Strongly Convex Functions with Applications to Prototype Selection}
\author[1]{Karthik S. Gurumoorthy \thanks{gurumoor@amazon.com}}
\author[2]{Amit Dhurandhar \thanks{adhuran@us.ibm.com}}
\affil[1]{Amazon Development Center, Bangalore, India}
\affil[2]{AI Foundations, IBM Research, New York, USA}

\date{ }

\begin{document}
\maketitle
\begin{abstract}
In this paper, we show that if the optimization function is restricted-strongly-convex (RSC) and restricted-smooth (RSM) -- a rich subclass of weakly submodular functions -- then a streaming algorithm with constant factor approximation guarantee is possible. More generally, our results are applicable to any monotone weakly submodular function with submodularity ratio bounded from above. This (positive) result which provides a sufficient condition for having a constant factor streaming guarantee for weakly submodular functions may be of special interest given the recent negative result \cite{weakstream} for the general class of weakly submodular functions. We apply our streaming algorithms for creating compact synopsis of large complex datasets, by selecting $m$ representative elements, by optimizing a suitable RSC and RSM objective function. Above results hold even with additional constraints such as learning non-negative weights, for interpretability \cite{proto,nmf}, for each selected element indicative of its importance. We empirically evaluate our algorithms on two real datasets: MNIST- a handwritten digits dataset and Letters- a UCI dataset containing the alphabet written in different fonts and styles. We observe that our algorithms are orders of magnitude faster than the state-of-the-art streaming algorithm for weakly submodular functions and with our main algorithm still providing equally good solutions in practice.
\end{abstract}

\section{Introduction}
Extracting compact synopses of large data sets or important features are a vital tool for summarizing, understanding, explaining and manipulating large datasets and large, complex machine learning models \cite{Kim16,proto}. Besides interpretability and human understanding, such synopses equally enable outlier detection, retaining information in lifelong learning systems, scaling deep learning, transfer learning and obtaining quick performance estimates for autoML systems \cite{automl}. These applications demand fast yet accurate and reliable algorithms for synopsis generation that can flexibly adapt to user and application demands and are robust to uncertainties in the data. These approaches can be unified as finding a subset $S$ out of a collection $V$ of items (data points, features, etc.) that maximize a scoring function $f(S)$. The scoring function measures the information, relevance and quality of the selection. The desiderata for the scoring function naturally imply notions of \emph{diminishing returns}: for any two sets $S \subset T \subset V$ and any item $i \notin T$, it holds that $f(S \cup \{i\}) - f(S) \geq f(T \cup \{i\}) - f(T)$. This is the definition of \emph{submodularity} \cite{fujishige05,lo83}.

In this paper, we provide two streaming algorithms for selecting such high value elements from data streams or large complex datasets. We also learn non-negative weights for each of them indicative of their importance. The non-negativity makes the weights more interpretable, as many domain experts find negative weights hard to interpret \cite{nmf,proto}. Our first streaming algorithm, ProtoBasic, is extremely efficient and for which we prove a constant factor approximation guarantee when the objective that it tries to maximize is restricted strongly convex (RSC) and restricted smooth (RSM) \cite{weaksub}, even with the additional non-negativity constraint on the importance weights. Functions that are RSC and RSM form a rich subclass of weakly submodular functions, including but not limited to ordinary least squares, generalized linear models, structured regularizers for matrix completion or any form of M-estimator \cite{weaksubInit, rsceg}. Loosely speaking, weakly submodular functions are close to being submodular but not quite and for which greedy algorithms lead to good solutions in the batch setting \cite{Nemhauser78}. Submodularity ratio \cite{weaksubInit} is a way of measuring this distance from submodularity.

In fact more generally, a constant factor bound can be shown for monotonic weakly submodular functions for whom the submodularity ratio can be bounded from above. This includes the RSC and RSM function class. This (positive) result which provides a sufficient condition for having a constant factor streaming guarantee for weakly submodular functions may be of special interest given the recent negative result \cite{weakstream} showing the absence of such a guarantee for the general class of weakly submodular functions. As an example and for the reader to obtain further insight we discuss the counter example given in \cite{weakstream} used to prove their negative result in the context of submodularity ratio, arguing that it cannot be bounded for that specific function.

Our second streaming algorithm, ProtoStream, is an enhancement of the first and is threshold based selecting elements with high incremental gain leading to a diverse selection which may not be the case with ProtoBasic. We provide theoretical arguments for which thresholds should be selected when running this algorithm.

We then empirically evaluate the efficacy of our algorithms for the prototype selection application \cite{proto}. We compare with the state-of-the-art streaming algorithm recently proposed for weakly submodular functions \cite{weakstream} in terms of performance and speed on two real datasets: MNIST- a handwritten digits dataset and Letters- a UCI dataset containing the alphabet written in different fonts and styles.

\section{Preliminaries}
Given a positive integer $n$, let $[n]:=\{1,...,n\}$ denote the set of the first $n$ natural numbers. 
\begin{definition}[Definition 1 (Submodularity Ratio):]
Let $L, S \subset [n]$ be two disjoint sets, and $f:[n]\rightarrow R$. The submodularity ratio \cite{weaksubInit} of L with respect to (w.r.t.) S is given by:
\begin{equation}
\label{subr1}
\gamma_{L,S} = \frac{\sum_{i\in S}\left(f(L\cup i)-f(L)\right)}{f(L\cup S)-f(L)}
\end{equation}
\eat{
The submodularity ratio of a set $U$ w.r.t. a positive integer $r$ is given by:
\begin{equation}
\label{subr2}
\gamma_{U,r} = \min\limits_{\substack{L,S: L \cap S=\emptyset \\ L\subseteq U;|S|\le r}}\gamma_{L,S}
\end{equation}
}
\end{definition}
The function $f(.)$ is submodular iff $\forall L,S$, $\gamma_{L,S}\ge 1$. However, if $\gamma_{L,S}$ can be shown to be bounded away from 0, but not necessarily $\ge 1$, then $f(.)$ is said to be weakly submodular.

\begin{definition}[Definition 2 (RSC and RSM):]
\label{def:RSCRSM}
A function $l:R^{n+}\rightarrow R$ is said to be restricted strong concave with parameter $c_\Omega$ and restricted smooth with parameter $C_\Omega$ \cite{weaksub} if $\forall \x,\y \in \Omega\subset R^{n+}$;

\begin{equation}
-\frac{c_\Omega}{2}\|\y-\x\|^2_2\ge l(\y)-l(\x)-\langle\nabla l(\x),\y-\x\rangle\ge-\frac{C_\Omega}{2}\|\y-\x\|^2_2.
\end{equation}
\end{definition}
We denote the RSC and RSM parameters on the domain $\Omega_m =\{ \x: \|\x\|_0 \leq m; \x \geq 0\}$ of all \emph{m-sparse non-negative} vectors by $c_{m}$ and $C_{m}$ respectively. We care about this non-negative orthant denoted by $R^{n+}$ because of our additional non-negativity constraint on the learned weights for each selected prototypes motivated from an interpretability standpoint. This is further explained in Section~ref{sec:experiments}. Also, let $\tilde{\Omega} = \{(\x,\y): \|\x-\y\|_0 \leq k\}$  with the corresponding smoothness parameter $\tilde{C}_k$.

\section{Problem Statement}
Given $n$ elements from an input space $X$, a constant $m << n$ independent of n, and a continuous function $l:R^{n+}\rightarrow \mathbb{R}$ with RSC and RSM properties, our objective is:
\begin{equation*}
\mbox{Maximize   } l(\w) \mbox{  s.t.  } \|w\|_0 \leq m \mbox{ and } \w \geq 0.
\end{equation*}
Defining a set function $f: [n] \rightarrow \mathbb{R}$ as 
\begin{equation}
\label{def:f}
f\left(L\right) \equiv \max\limits_{\w: supp(\w) \in L} l\left(\w\right)
\end{equation}
for a set $L \subset [n]$ where $supp(\w) = \{j: \w_j \geq 0\}$, our goal is to find that set $L = L^{\ast}$ that maximizes $f(.)$ subject to the cardinality constraint that $\left|L^{\ast} \right| \leq m$. Note that $f(.)$ is monotonic as if $L_1 \subseteq L_2$ then $f\left(L_2\right) \geq f\left(L_1\right)$. Hence, without loss of generality we assume that $f(\emptyset) = 0$. Given a set $L$, the point at which $l(.)$ attains maximum with the support in $L$ is represented by $\bzeta^{\left(L\right)}$. 

Easy to see that explicitly computing $L^{\ast}$ is an NP-complete problem. In this work, we develop a fast streaming algorithm that closely approximates $f\left(L^{\ast}\right)$ even for the \emph{worst case streaming order} of the $n$ elements. To this end, we show later that when $l(.)$ is RSC and RSM, then it is possible to have a \emph{constant factor streaming algorithm} even for the worst case streaming order. More generally, we establish that if the submodularity ratio for any weakly submodular monotonic set function $f(.)$ is bounded from above, then a streaming algorithm with constant approximation guarantee exists as stated in Theorem~\ref{thm:constappboundedgamma}.

\section{Related Work}

As mentioned before subset selection especially based on submodularity has wide applications in understanding, summarizing and manipulating large datasets \cite{fujishige05,lo83,Kim16} given that it is possible to obtain tractable algorithms with constant factor guarantees. In fact, it is known that even in the streaming setting \cite{substream} constant factor algorithms are possible for submodular functions.

Recently, it was shown \cite{weakstream} though that for the larger class of weakly submodular functions \cite{weaksubInit} no constant factor algorithm can exist in the streaming setting. This was a surprising result given that for the batch setting it has been known for a while that such approximation algorithms do exist \cite{Nemhauser78}. 

In this work we propose streaming algorithms for a rich subclass of weakly submodular functions \cite{weaksub} namely those that are RSC and RSM. Efficient batch algorithms for the same were proposed in \cite{proto,weaksub}. In fact, the focus on interpretability through learning non-negative weights was highlighted in \cite{proto}. Our work thus shows that a constant factor streaming algorithm is possible for RSC and RSM weakly submodular functions or more generally for weakly submodular functions for whom the submodularity ratio can be bounded from above even with having to learn non-negative weights for the selected elements indicative of their importance from an interpretability standpoint \cite{nmf,proto,sproto}. This thus provides a sufficient condition, which includes a rich enough subclass of weakly submodular functions, for obtaining such a guarantee and is interesting in light of the recent result \cite{weakstream}.
 
\begin{algorithm}[t]
    \caption{ProtoBasic}
    \label{protobasic}
\begin{algorithmic}
\STATE \textbf{Input:} sparsity level $m$, elements $X$ and function $l(.)$
\STATE $L=\emptyset$, $\bzetaL=\mathbf{0}$
\FOR{each new element with index $j$}
\STATE \textbf{if} $|L|< m$ \textbf{then} $L = L \cup \{j\}$
\STATE \textbf{Else} 
\STATE ~~~~~~~Compute $g_j = \nabla l_j(\mathbf{0})$, $g_{min} = \min\limits_{i\in L} \nabla l_i(\mathbf{0})$ and $k = \argmin\limits_{i\in L} \nabla l_i(\mathbf{0})$
\STATE ~~~~~~~\textbf{if} $g_j > g_{min}$ \textbf{then} Update $L=L\setminus\{k\}\cup \{j\}$ \textbf{end if}
\STATE \textbf{end if}
\ENDFOR 
\STATE $\bzetaL=\argmax\limits_{\w:supp(\w)\in L,\w\ge 0} l(\w)$
\RETURN $L$, $\bzetaL$ 
\end{algorithmic}
\end{algorithm}

\begin{algorithm}[t]
    \caption{ProtoStream}
    \label{protostream}
\begin{algorithmic}
\STATE \textbf{Input:} sparsity level $m$, elements $X$, function $l(.)$, $ \epsilon$
\STATE $v=-\infty$, $L=\emptyset$, $L_{\tau}=\emptyset$, $\bzetaLt =\mathbf{0}$ \COMMENT{$L_{\tau}$ is the set corresponding to threshold $\tau$.}
\FOR{each new element with index $j$}
\STATE \textbf{if} $\nabla l_j(\mathbf{0})\ge v$
\textbf{then} $v=\nabla l_j(\mathbf{0})$, $L=\{j\}$ and $\rho = \nabla l_i^+(\mathbf{0})^2$ \textbf{end if}
\FOR{threshold $\tau\in O_{\rho}=[\frac{\rho}{2m},\frac{\rho m}{2}]$ at geometric sequence with ratio $(1+\epsilon)$}
\IF{$|L_{\tau}|< m $ and $\nabla l_j(\bzetaLt)\ge \sqrt{\frac{2\tau}{m}}$}
\STATE $L_{\tau} = L_{\tau} \cup \{j\}$ and $\bzetaLt=\argmax\limits_{\w:supp(\w)\in L_{\tau},\w\ge 0} l(\w)$
\ENDIF
\ENDFOR 
\ENDFOR
\STATE $\bzetaL=\argmax\limits_{\w:supp(\w)\in L,\w\ge 0} l(\w)$
\STATE \textbf{if} $l(\bzetaL)\le \max\limits_{L_{\tau}}l(\bzetaLt)$ \textbf{then} set $\eta=\argmax\limits_{\tau}l(\bzetaLt)$, $L=L_{\eta}$ and $\bzetaL=\bzeta^{(L_{\eta})}$ \textbf{end if}
\RETURN $L$, $\bzetaL$ 
\end{algorithmic}
\end{algorithm}

\section{Methods and Results}

In this section we propose two streaming algorithms, a simple one and an enhanced threshold based one. We show based on our first algorithm that it is possible to obtain a constant factor bound for RSC and RSM functions. Also more generally, the constant factor bound can be shown for any monotonic weakly submodular function with submodularity ratio bounded from above. Here we also discuss the counter example given in \cite{weakstream} in the context of submodularity ratio. We then describe our second threshold based algorithm which is an enhancement of the first and that adds elements based on high incremental gain and is thus likely to select diverse elements leading to potentially better performance in practice. We provide a (theoretical) discussion here of what thresholds should be considered when running this algorithm.

\subsection{Algorithmic Description for ProtoBasic}

Algorithm \ref{protobasic}, ProtoBasic, is the first streaming algorithm we propose. The algorithm is quite simple where we maintain only one active solution set $L$ making it extremely fast. Moreover, only function gradient evaluations are required for deciding on each new element, rather than function evaluations as in \cite{weakstream} adding to its scalability.

The algorithm first proceeds by selecting the first $m$ elements. Then for every subsequent element it checks the value of adding that element to the empty set based on the function gradient. If this value is higher than the minimum value amongst the elements that have been currently selected, then we replace this minimum value element with the current one. The minimum value element can be accessed efficiently possibly using a min heap data structure. Finally, the optimal weights can be computed for the selected set.

\subsection{Theoretical Guarantees}
We first based on Lemmas \ref{lemma:lowerBound} and \ref{lemma:upperBound} for any RSC and RSM function show a constant factor bound for ProtoBasic. We then show how RSC and RSM implies bounds on the submodularity ratio and how a bounded submodularity ratio can also lead to algorithms with constant factor guarantees. Complete proofs can be found in Appendix~\ref{sec:proofs} .

\begin{lemma}
\label{lemma:lowerBound}
Let $\tilde{C}_k$ be the RSM constant for any two vectors $\x$ and $\y \in \R^{b^+}$ where $\|\x - \y\|_0 \leq k$. Then for any two sets $L$ and $S$ with $L \cap S = \emptyset$, $|S| \leq k$ and $\nabla l_{S}^+\left(\bzetaL \right)=\max \left(\gradlSbzetaL ,0 \right)$ we have,
\begin{equation}
\lbzetaLS - \lbzetaL \geq \frac{1}{2 \tilde{C}_k} \left\|\nabla l_{S}^+\left(\bzetaL \right) \right\|^2 \nonumber.
\end{equation}
\end{lemma}
\eat{*****\begin{proof}
Let $\bonej$ be a vector with a value one only at the $j^{th}$ coordinates and zero elsewhere.  For all $\alpha_j \geq 0$, define $\byS = \bzetaL + \sum\limits_{j \in S} \alpha_j \bonej$.  As $\bzetaLS$ is the optimal point for $f\left(L \cup S \right)$ we have
\begin{align}
\label{eq:lowerbound}\lbzetaLS - \lbzetaL &\geq l\left(\byS\right) - \lbzetaL \geq  \left\langle\nabla \lbzetaL,  \sum\limits_{j \in S} \alpha_j \bonej \right\rangle - \frac{\tilde{C}_k}{2} \sum\limits_{j \in S} \alpha_j^2.
\end{align}

Maximizing w.r.t. each $\alpha_j$, we get $\alpha_j = \frac{\gradljpbzetaL}{\tilde{C}_k}$ where $\gradljpbzetaL = \max \left(\gradljbzetaL ,0 \right)$. Substituting these values of $\alpha_j$ in (\ref{eq:lowerbound}) gives us the required lower bound, namely
\begin{align}
\label{eq:lowerBfinal}
\lbzetaLS - \lbzetaL \geq \frac{1}{2 \tilde{C}_k} \left\|\nabla l_{S}^+\left(\bzetaL \right) \right\|^2.
\end{align}
\end{proof}****}
\begin{proof}[Proof Sketch]
Based on definition of RSM and evaluating the KKT conditions for optimality we get the necessary lower bound.
\end{proof}

\begin{lemma}
\label{lemma:upperBound}
Let $c_k$ be the RSC constant for any two $k$ sparse vectors $\x$ and $\y \in \R^{b^+}$.  Then for any two sets $L$ and $S$ with $L \cap S = \emptyset$, $|L| + |S| = k$ and $\nabla l_{S}^+\left(\bzetaL \right)=\max \left(\gradlSbzetaL ,0 \right)$
 we have,
\begin{equation}
\lbzetaLS - \lbzetaL \leq \frac{1}{2 c_k} \left\|\nabla l_{S}^+\left(\bzetaL \right) \right\|^2 \nonumber.
\end{equation}
\end{lemma}

\begin{proof}[Proof Sketch]
Based on definition of RSC and evaluating the KKT conditions for optimality we get the necessary upper bound.
\end{proof}

\begin{theorem}[Constant factor guarantee for RSC and RSM functions]
\label{constrsc}
Consider a function $l : R^{n+}\rightarrow R$ with RSC and RSM parameters $c_m$ and $\tilde{C}_m$ respectively and let $f(.)$ be a set function defined as in (\ref{def:f}). If $S$ is the solution of ProtoBasic and $L^{\ast}$ is the optimal set of size $m$, then for $\kappa=\frac{c_m}{\tilde{C}_m}$ we have
\begin{equation*}
\label{eq:const}
f\left(S\right) \geq \kappa  f\left(L^{\ast}\right). 
\end{equation*}
\end{theorem}
\begin{proof}[Proof Sketch]
First setting $L=\emptyset$ in lemma \ref{lemma:lowerBound} and then setting $S=L^{\ast}$ and $L=\emptyset$ in lemma \ref{lemma:upperBound} we get the constant factor bound.
\end{proof}

\begin{lemma}[Bounded submodularity ratio $\gamma$]
Let $f(.)$ be a set function defined as in (\ref{def:f}) where $l(.)$ is RSC and RSM. Then for any two disjoint sets $L$ and $S$ we have,
\begin{equation*}
\label{eq:submodbounds}
\frac{c_{|L|+|S|}}{\tilde{C}_1} \leq \gamma_{L,S} \leq \frac{\tilde{C}_{|S|}}{c_{|L|+1}}. 
\end{equation*}
\end{lemma}
\begin{proof}[Proof Sketch]
Using inequalities in lemmas \ref{lemma:lowerBound} and \ref{lemma:upperBound} we can bound the submodularity ratio for any RSC and RSM function as above.
\end{proof}

\begin{theorem}[Constant factor guarantee for functions with bounded $\gamma$]
\label{thm:constappboundedgamma}
Let $f(.)$ be a monotonic weakly submodular function with the property that any set $Z$ of cardinality $m$ has a bounded submodularity ratio, i.e., $r_m\le\gamma_{\emptyset,Z}\le R_m$ where $r_m$ and $R_m$ are positive constants independent of $Z$ and depends only on $m$. Then the set $S$ containing the $m$ elements with the highest singleton $f(.)$ values computable in a streaming setting (say by using min heaps) satisfies,
\begin{equation*}
\label{eq:const2}
f\left(S\right) \geq \kappa  f\left(L^{\ast}\right) \text{ where }\kappa= \frac{r_m}{R_m}
\end{equation*}
where $L^{\ast}$ is the optimal size $m$ solution at which $f\left(L^{\ast}\right)$ attains maximum.
\end{theorem}
\begin{proof}[Proof Sketch]
The result follows from the inequalities that ensue given the fact that $\forall j \in S$; $ f(\{j\}) \geq f(\{p\})$; $p \notin  S$. 
\end{proof}

\subsection{Impossibility Result and Submodularity Ratio}

We now briefly describe how the submodularity ratio of the weakly submodular function constructed in \cite{weakstream} to show the impossibility result \emph{cannot} be bounded from above and thus does not contradict our results. Moreover, it provides insight into the connections between the two. As considered in \cite{weakstream}, for any set $S$ define the functions $u(S) = |S \cap U|$ and $v(S) = |S \cap V|$ using the base elements $U = \{u_i\}_{i=1}^k$ and $V=\{v_i\}_{i=1}^k$. An impossibility result is shown for the set function $f_k(S) = \min\{2u(S) +1, 2 v(S)\}$. Letting $L = \emptyset$ and $S = V$ we find $\gamma_{L,S} = \frac{\sum\limits_{j \in S} f_k(\{j\})}{f_k(S)}$. For any singleton set $\{j\} \subset V$, $u(\{j\}) = 0$ and $ v(\{j\}) = 1$ implies $f_k(\{j\}) = 1$. Further $f_k(S) = 1$ as $u(S) = 0$. Hence $\gamma_{L,S} = |S| = k$ grows with $k$ which can be made large enough to violate any upper bound and thereby engendering the impossibility result.

\subsection{Algorithmic Description for ProtoStream}
Algorithm \ref{protostream}, ProtoStream, unlike ProtoBasic is threshold based. It maintains multiple candidate sets of elements in parallel corresponding to thresholds in the range $O_{\rho} =[\frac{\rho}{2m}, \frac{\rho m}{2}]$ at intervals of $(1+\epsilon)$ for an user input $\epsilon \in (0,1)$. Here $\rho = \left[\gradlppbzero\right]^2$ where $p$ is the element such that $\gradlpbzero \geq \gradlibzero$ among all the encountered elements $j$. The total number of candidates sets that are simultaneously maintained is $O\left(\frac{\log m}{\epsilon}\right)$ requiring a total space of $O\left(\frac{m \log m}{\epsilon}\right)$ independent of $n$. Value $\rho$ depends on the highest gradient element $p$ encountered thus far which is also one of the candidate sets. Those sets are updated for which the incremental gain in adding the new element based on its gradient is greater than $\sqrt{\frac{2\tau}{m}}$, where the $\tau$ are the  thresholds in $O_{\rho}$. Notice that the incremental gain is a constant that does \emph{not depend} on $\gamma$ or RSC and RSM parameters of the objective function and is thus easily computable. Eventually, the set along with its corresponding weights that has the highest value of $l(.)$ is chosen as the final solution. Lemma~\ref{lemma:cardinalityMet} gives a lower bound for the set function evaluated at the set $L_{\tau}$ containing $m$ elements corresponding to a threshold $\tau$.

\begin{lemma}
\label{lemma:cardinalityMet}
If the set $L_{\tau}$ for the threshold $\tau$ has cardinality $m$ then $f\left(L_{\tau}\right) \geq \frac{\tau}{\tilde{C}_1}$.
\end{lemma}
\begin{proof}[Proof Sketch]
The result follows from Lemma \ref{lemma:lowerBound} and that we add an element $j$ only if $\nabla l_j(\bzetaLt)\ge \sqrt{\frac{2\tau}{m}}$
\end{proof}

\subsection{Choosing Thresholds for ProtoStream}
Recall that the thresholds are searched in the interval $O_{\rho} = [\frac{\rho}{2m}, \frac{\rho m}{2}]$ where the interval length is \emph{independent} of the RSC and RSM parameters and hence readily available. The upper bound $\frac{\rho m}{2}$ on the range of $\tau$ is chosen to guarantee that for any new element $j$, all candidate sets $L_{\tau}$ to which $j$ must be appended when its incremental gain exceeds $\sqrt{\frac{2\tau}{m}}$ are considered and no already seen elements are overlooked that should have been taken for the set $L_{\tau_{new}}$ when instantiating a new $\tau_{new} > \frac{\rho m}{2}$. This is because when $\tau$ is chosen from $O_{\rho}$, every element $j$ that satisfies the threshold criteria to be a part of $L_{\tau}$ will appear on or after $\tau$ is instantiated and never before, as for any past element $j$, $\left[\gradljpbzero\right]^2 \leq \rho < \frac{2\tau_{new}}{m}$ where $\tau_{new}$ is the new value of $\tau$ that may be instantiated after seeing $j$. Ergo, $j \notin L_{\tau_{new}}$. The following insight is useful in motivating our choice for the lower range of $O_{\rho}$. Setting $S = L^{\ast}$ and $L = \emptyset$ in Lemma~\ref{lemma:upperBound} we get
\begin{equation}
\label{eq:OPTUpperBoundWithMaxGrad}
f\left(L^{\ast}\right) \leq \frac{1}{2c_m}\sum\limits_{j \in L^{\ast}} \left[\gradljpbzero\right]^2 \leq \frac{\rho m}{2c_m}
\end{equation}
implying that $\frac{\rho m}{2} \geq c_m f\left(L^{\ast}\right)$. Hence we choose the lower range of $O_{\rho}$ to be the value that lower bounds $c_m f\left(L^{\ast}\right)$. Setting $S$ to be the singleton set $\{p\}$ which has the maximum gradient at $\mathbf{0}$ and $ L = \emptyset$ in Lemma~\ref{lemma:lowerBound} we have
\begin{equation}
\label{eq:OPTLowerBoundWithMaxGrad}
\frac{c_m \rho}{2 \tilde{C}_1} \leq c_m f(\{p\}) \leq c_m f\left(L^{\ast}\right).
\end{equation}

Let us first consider the case where the number of chosen prototypes $m$ is so few that $\frac{c_m}{\tilde{C}_1} \leq \frac{1}{m}$. From Lemma~\ref{lemma:lowerBound} and the inequality in (\ref{eq:OPTUpperBoundWithMaxGrad}) we find
\begin{equation*}
f(\{p\}) \geq \frac{\rho}{2 \tilde{C}_1} \geq \frac{c_m f\left(L^{\ast}\right)}{\tilde{C}_1 m} \geq \frac{c_m^2 f\left(L^{\ast}\right)}{\tilde{C}_1^2}.
\end{equation*}
Hence by just opting for the singleton set $\{p\}$, we obtain a constant factor approximation. In the more interesting case where $\frac{c_m}{\tilde{C}_1} \geq \frac{1}{m}$, (\ref{eq:OPTLowerBoundWithMaxGrad}) implies that $c_m f\left(L^{\ast}\right) \geq \frac{\rho}{2 m}$. Hence we set the range to be $O_{\rho} = [\frac{\rho}{2 m}, \frac{\rho m}{2}]$. Note that for a value $\tau \in O_{\rho} \geq c_m f\left(L^{\ast}\right)$, if $\left|L_{\tau} \right| =m$, then in accordance with Lemma~\ref{lemma:cardinalityMet} we will have $f\left(L_{\tau}\right) \geq \frac{c_m  f\left(L^{\ast}\right)}{\tilde{C}_1}$, resulting in a better constant approximation factor compared to $\frac{c_m}{\tilde{C}_m}$ derived for ProtoBasic as $\tilde{C}_1 \geq \tilde{C}_m$.

\begin{figure*}[t]
  \begin{center}
    \begin{tabular}{ccc}
      \includegraphics[width=0.33\linewidth]{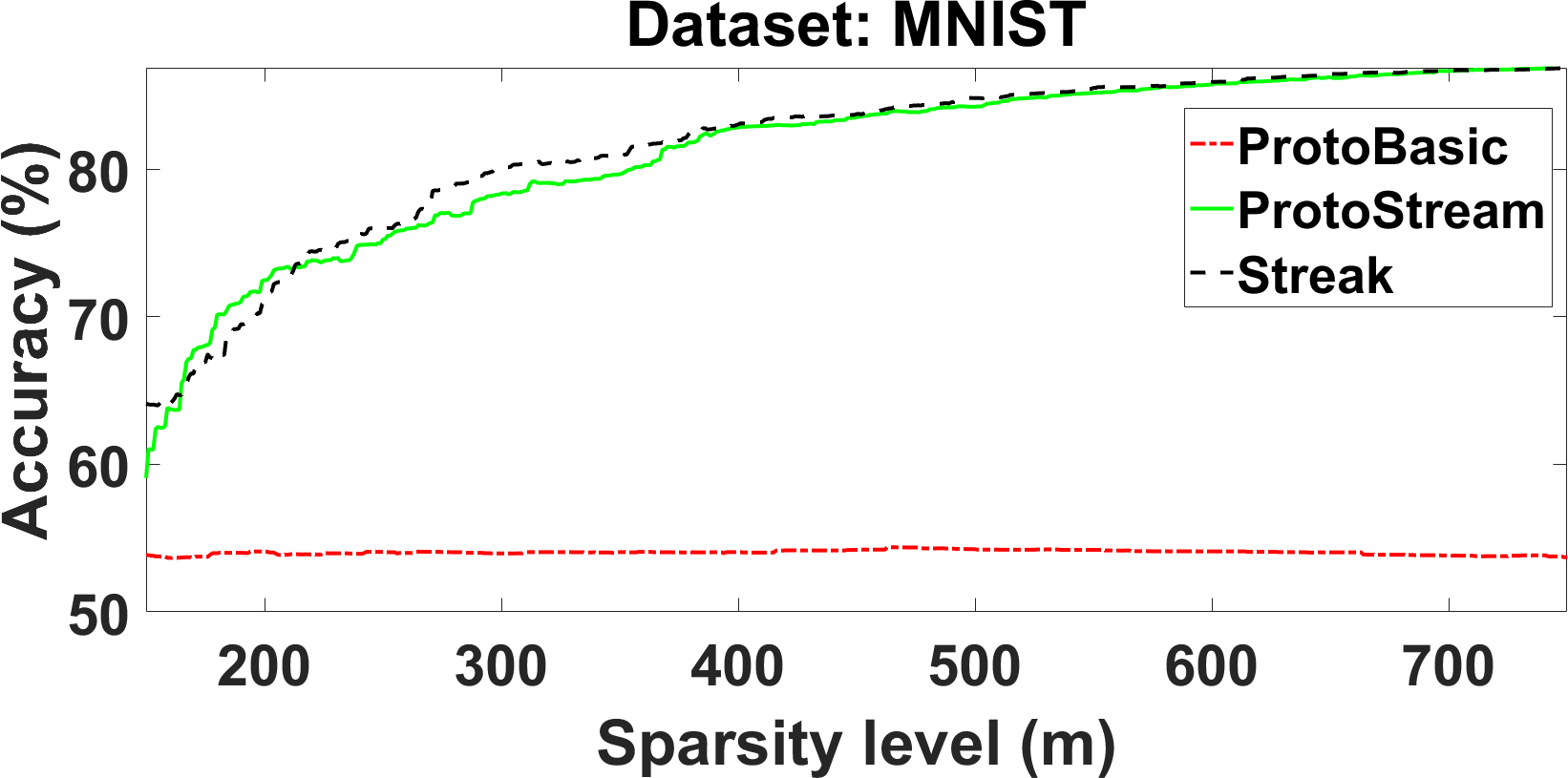} &      \includegraphics[width=0.33\linewidth]{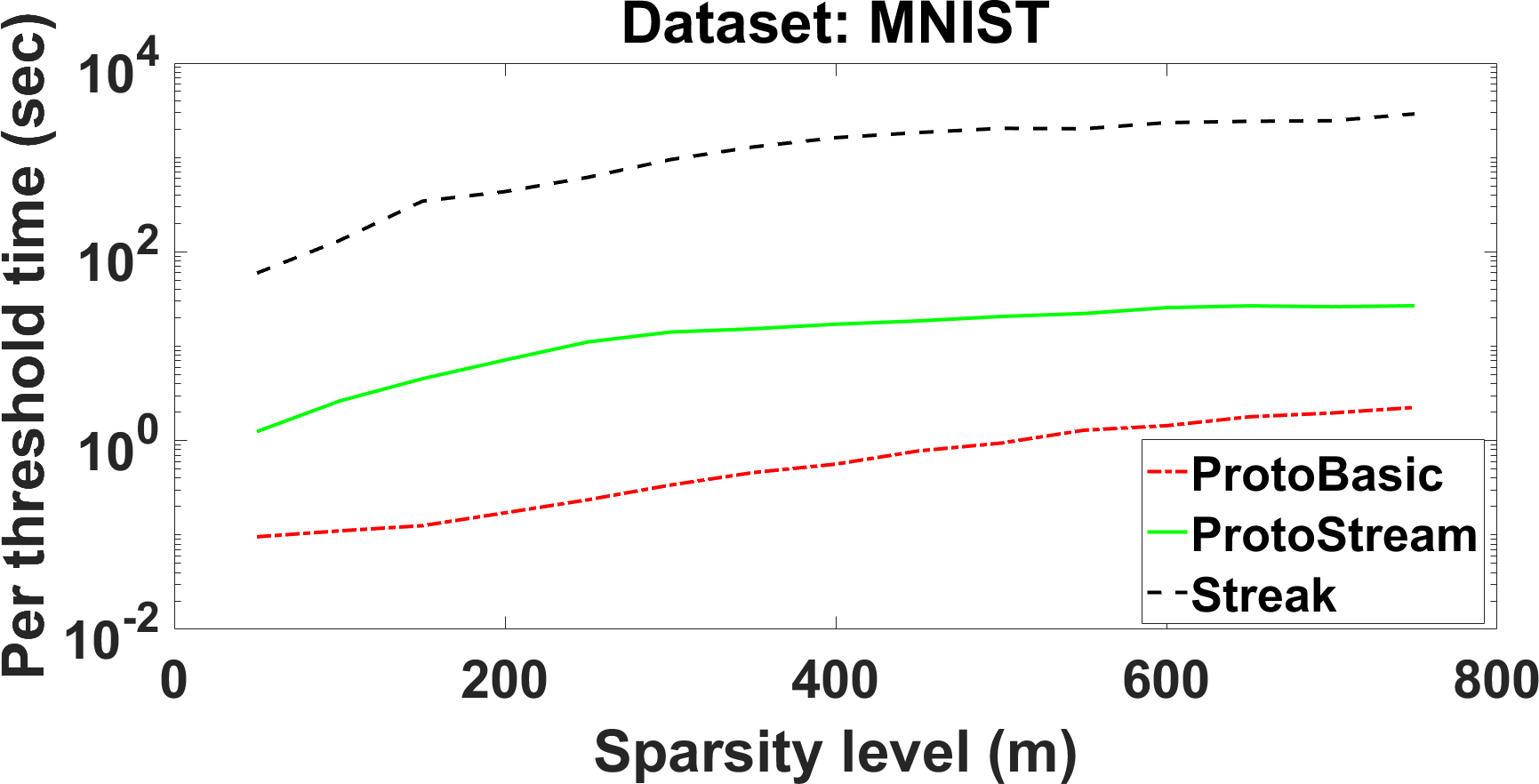} &
      \includegraphics[width=0.33\linewidth]{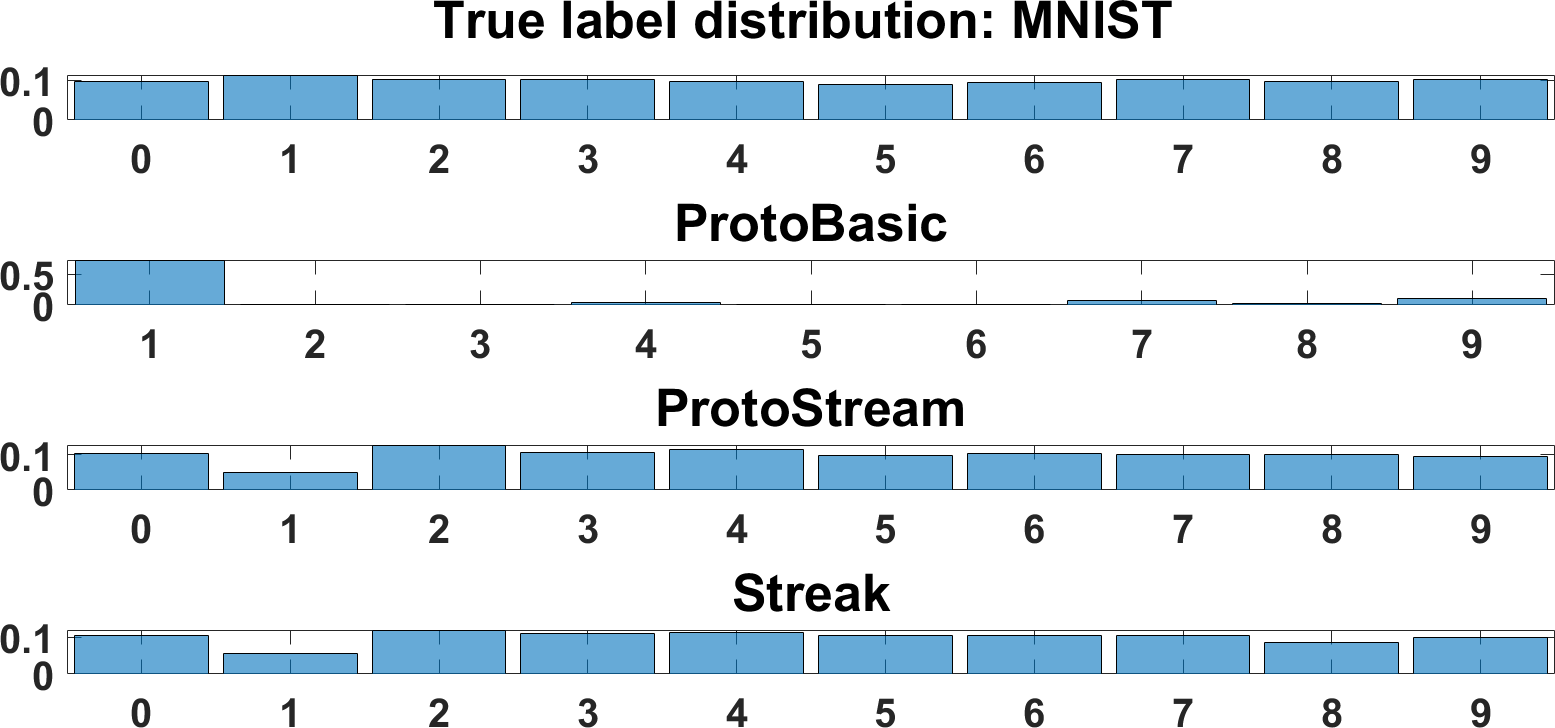}
     \end{tabular}
  \end{center}
  \caption{We observe the performance of the methods on MNIST for different $m$. The left Figure depicts the classification accuracy. The center figure depicts the (per threshold) running time. The right figure depicts the label distribution of the selected prototypes.}
  \label{mnist}
\end{figure*}
\begin{figure*}[t]
  \begin{center}
    \begin{tabular}{ccc}
      \includegraphics[width=0.33\linewidth]{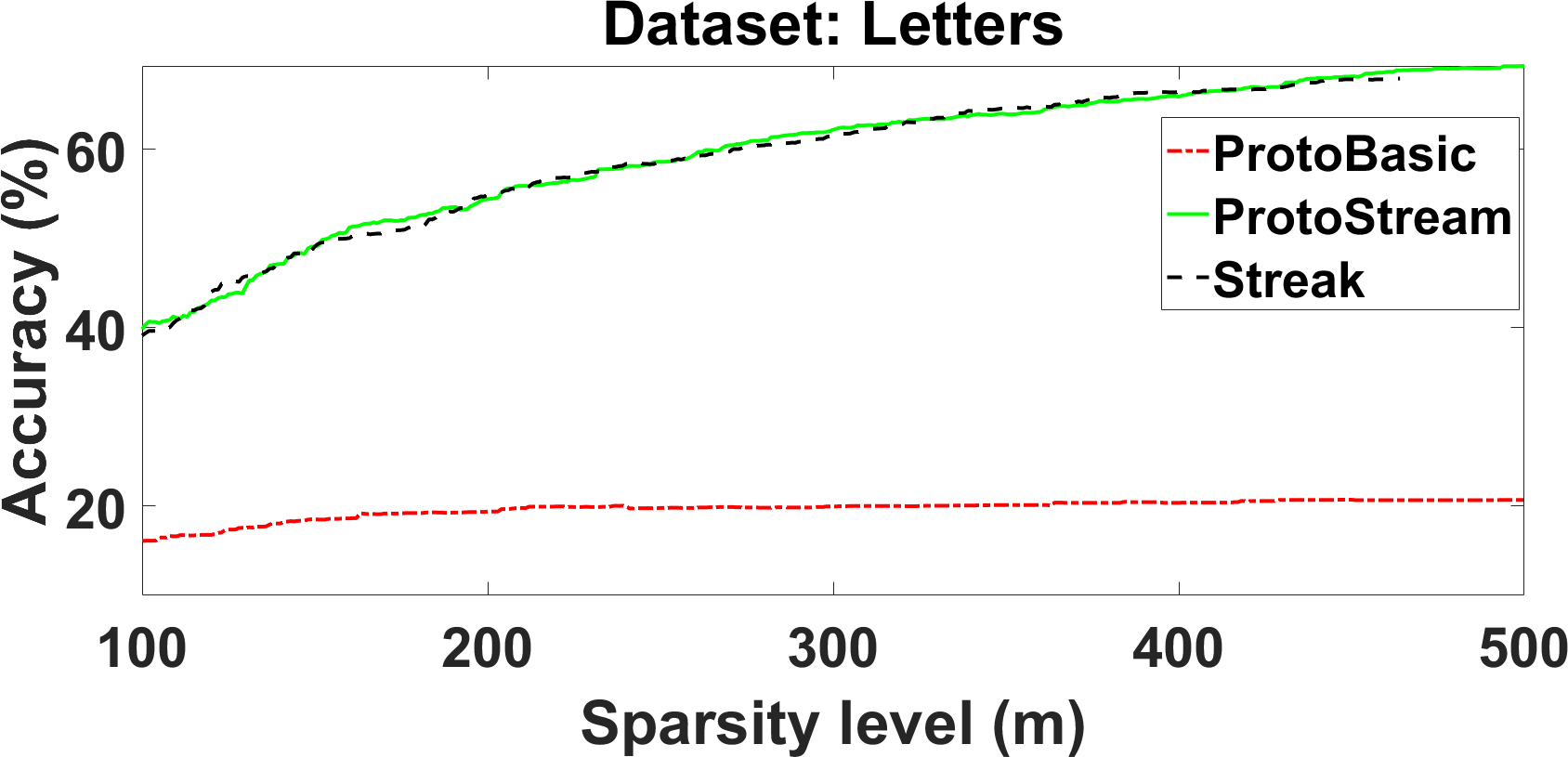} &      \includegraphics[width=0.33\linewidth]{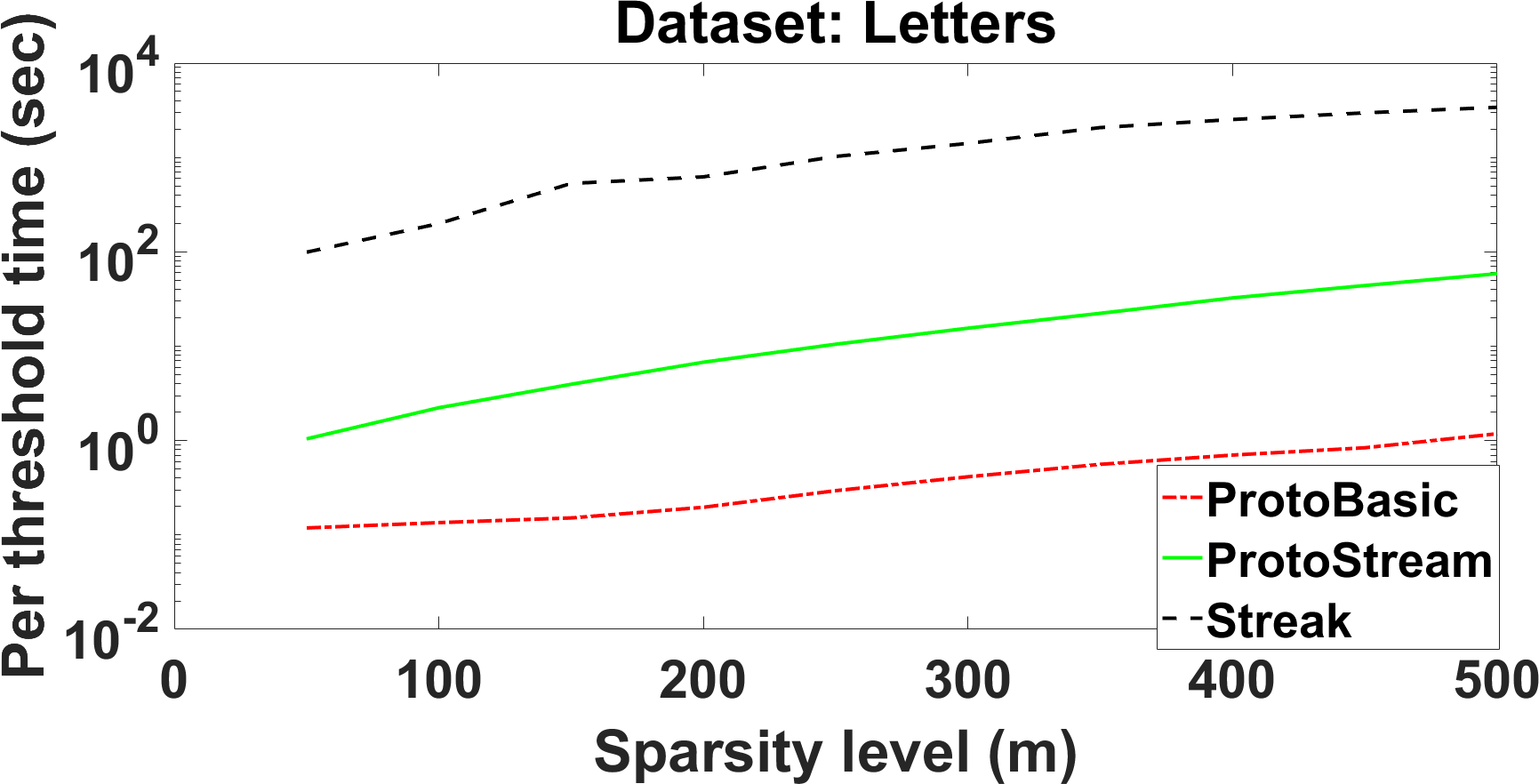} &
      \includegraphics[width=0.33\linewidth]{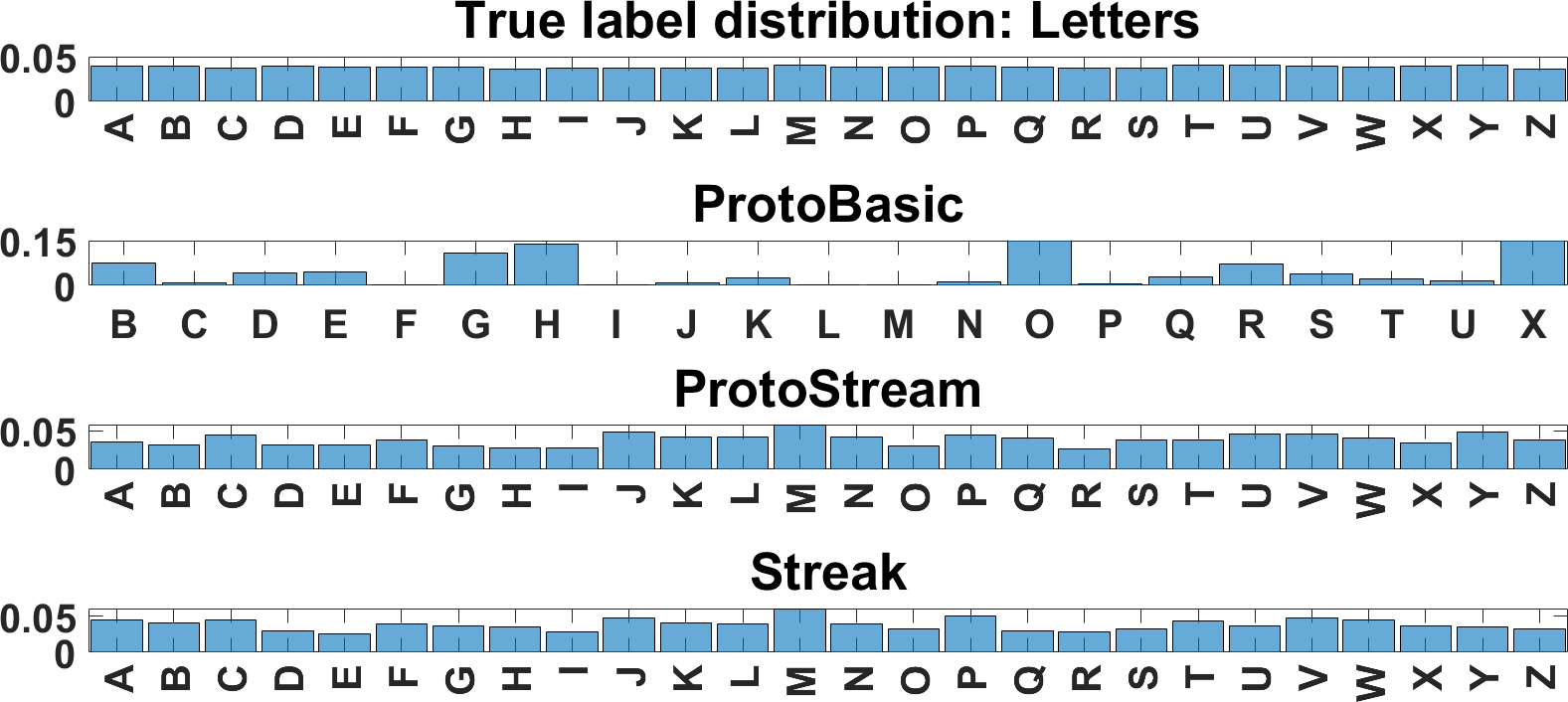}
     \end{tabular}
  \end{center}
  \caption{We observe the performance of the methods on Letters for different $m$. The left figure depicts the classification accuracy. The center figure depicts the (per threshold) running time. The right figure depicts the label distribution of the selected prototypes.}
  \label{letters}
\end{figure*}

\section{Experiments}
\label{sec:experiments}
We now empirically investigate the performance of our algorithms relative to the state-of-the-art Streak algorithm \cite{weakstream} on two real datasets MNIST \cite{mnist} and Letters \cite{letters}. We extract compact synapses on the fly for these datasets of size $n$ by selecting a maximum of $m << n$ prototypes obtained by maximizing the following cost function, which is a reformulation of maximum mean discrepancy metric and has been successfully used to select prototypes in the batch setting \cite{proto,Kim16}:
\begin{equation}
\label{eq:l} 
\mbox{Maximize   }l\left(\w\right) =  \w^T \bmu - \frac{1}{2} \w^T K \w \mbox{   s.t.  } \|\w\|_0 \leq m \mbox{ and } \w \geq 0.
\end{equation}
Here $K$ is the positive definite Kernel matrix with entries $K_{i,j} = k(\x_i,\x_j)$ where $k$ is appropriately chosen kernel function to define the inner products between data samples. The entries of the vector $\boldsymbol \mu$ contains the mean inner product of a data sample with the rest and is defined as $\mu_{j} = \frac{1}{n} \sum\limits_{i=1}^{n} k(\x_i,\x_j); \forall j$. An empirical estimate of $\bmu$ is maintained based on ideas described in \cite{substream} in the experiments. The vector $\w=[w_{1},...,w_{n}]^T$ are the non-negative weights with utmost $m$ entries non-zero which are indicative of the importance of the corresponding prototypes. It was shown in \cite{proto} that the function in equation \ref{eq:l} is RSC and RSM and the corresponding set function $f(.)$ defined as in equation \ref{def:f} is weakly submodular \emph{even} with the non-negativity constraint on the weights. When all weights are set to $1/m$ and only the support set $L$ is unknown, the set function in (\ref{def:f}) can be shown to be (strongly) submodular \cite{Kim16} and for these class of functions, streaming algorithms with constant factor guarantees are developed in \cite{substream}. As \cite{proto} describes in detail the usefulness of having non-equal weights, we consider the more general setting here and apply our streaming algorithms for the same. In all the experiments we use a Gaussian kernel for $k(.,.)$ whose width is found through cross-validation and set $\epsilon = 0.4$ as smaller values didn't improve the objective by much, although significantly slowed down Streak.

For both MNIST and Letters, a (global) 1-nearest neighbor (1-NN) classifier \cite{Kim16} was used to evaluate the efficacy of the selected prototypes. Since the learned weights and the distance metric in 1-NN classification are not the same scale, we performed the standard 1-NN classification based on the top $m$ prototypes selected based on largest weights. 

Additional experiments where the test set is split into multiple target datasets containing only (examples of) a single digit/alphabet, while the training or source dataset remains the same, and we want to evaluate the adaptability of the algorithms to such heavily skewed test distributions, are given in Appendix~\ref{sec:skeweddistributionexp}. We observe in such settings that ProtoBasic is in fact the method of choice.

\begin{table}[t]
\caption{Below we see the total run times (in seconds) and the maximum objective value for $l(w)$ (equation \ref{eq:l}) achieved by the different methods. Best results are highlighted in bold.\label{table:summary}}
\noindent \centering{}%
\begin{tabular}{|c|c|c|c|c|}
\hline
Method&\multicolumn{2}{|c|}{MNIST (m=750)}&\multicolumn{2}{|c|}{Letters (m=500)}\\
\cline{2-5}
&Total run time (s)&Max. $l(w)$ value&Total running time (s)&Max. $l(w)$ value\\
\hline
ProtoBasic&\textbf{2}&0.0587&\textbf{1}&0.0395\\
\hline
ProtoStream&1482&\textbf{0.0705}&338&\textbf{0.0468}\\
\hline
Streak&108341&\textbf{0.0705}&52025&\textbf{0.0468}\\
\hline 
\end{tabular}
\end{table}

\subsection{MNIST}
The MNIST dataset consists of 70000 (60K+10K) handwritten digits. We use the set of size 10000, as the base set from which we choose up to 750 prototypes---since after this the gain to our objective \ref{eq:l} was incremental---and then evaluate it on the remaining 60000 using it as a test set.

We observe in Figure \ref{mnist} (left) that the performance of both Streak and ProtoStream in terms of classification accuracy on the test set are very similar across different values of $m$. ProtoBasic is significantly worse and the reason for this is the lack of diversity in its chosen prototypes as visualized in Figure \ref{mnist} (right). In this plot, we see that the distribution of the 10 digits in the base set is almost uniform, and both Streak and ProtoStream are able to reasonably recover this, however, ProtoBasic ends up selecting just a few digits. This is because ProtoBasic chooses the prototypes only based on their gradient values computed at $\mathbf{0}$, and is non-incremental in the sense that subsequent choices do not depend on which ones have been chosen thus far and hence is unable to create a diverse prototype set. However, both ProtoStream and Streak are incremental methods as the incremental gain for an incoming element depends on the current content of the sets.

In Figure \ref{mnist} (center), we see the main benefit of our methods. We plot the per threshold times as parallelized implementations may be possible for maintaining the different sets and so a comparison on this metric is important. In Table \ref{table:summary}, we see the total run times for a serial implementation of these methods. In both cases we see that \emph{ProtoStream is approximately 2 orders of magnitude faster than Streak}. Moreover, in Table \ref{table:summary} we observe that ProtoStream achieves the same quality solution as Streak, given that the maximum objective value (of equation \ref{eq:l}) is identical for both of them.

The reason for such a wide computational gap is that our algorithms only require gradient evaluations which are about $O(m)$ for each new instance, while Streak performs function evaluations which are $O(m^3)$ for each new instance as (\ref{def:f}) is a quadratic optimization problem. Moreover, while ProtoStream does only $m << n$ function evaluations to recompute the weights $\bzetaLt$ after the addition of an instance to the set $L_{\tau}$, Streak performs $n$ function evaluations per threshold as computing the incremental gain for every element requires such an evaluation.

\subsection{Letters}
The Letters dataset is a UCI repository dataset consisting of 20000 instances of the 26 letters in the alphabet written in 20 different fonts and 5 different styles. There are 16 attributes which encompass statistical moments and edge counts when scanning these letter images in different directions. Typically, the first 16000 instances are used for training and the remaining 4000 are used as test. We selected up to 500 prototypes from the base set of 16000 since the gain based on (\ref{eq:l}) after that was marginal. The selected prototypes were then used to classify the other 4000 using 1-NN classifier. In Figure \ref{letters} (left) we see that the accuracy of ProtoStream is almost indistinguishable from Streak and at times superior for some values of $m$. ProtoBasic, again performs inferiorly due to lack of diversity as elucidated above and is validated in Figure \ref{letters} (right). We again observe in Figure \ref{letters} (center) and Table \ref{table:summary} that our algorithms are orders of magnitude faster than Streak as they do not require evaluation of set function in (\ref{def:f}) for every new instance, albeit that ProtoStream still achieves the same quality (i.e. same max objective value) solution as Streak.

More experiments showcasing the diversity of our selection across fonts and stroke styles are given in Appendix~\ref{sec:fontstrokeexp}.

\section{Discussion}
In summary, we described sufficient conditions for obtaining a constant factor streaming algorithm for weakly submodular functions. Our conditions cover a rich class of functions namely those that are RSC and RSM. As a more general result, we established that any monotonic weakly submodular function with bounded submodularity ratio from above has a streaming algorithm with constant approximation guarantees. We developed an extremely fast threshold free algorithm and a high performing threshold based algorithm that is still orders of magnitude faster than the state-of-the-art at least for quadratic functions over several variables and also closely matches the latter in practical performance. 
In the future, we would like to study how much our conditions can be relaxed to bridge the gap between necessity and sufficiency for the rich class of weakly submodular functions.
\appendix
\section{Proofs}
\label{sec:proofs}
\subsection{Proof of Lemma 5.1}
\begin{proof}
Let $\bonej$ be a vector with a value one only at the $j^{th}$ coordinates and zero elsewhere.  For all $\alpha_j \geq 0$, define $\byS = \bzetaL + \sum\limits_{j \in S} \alpha_j \bonej$.  As $\bzetaLS$ is the optimal point for $f\left(L \cup S \right)$ we have
\begin{align}
\lbzetaLS - \lbzetaL &\geq l\left(\byS\right) - \lbzetaL \nonumber \\
 \label{eq:lowerbound}
&\geq  \left\langle\nabla \lbzetaL,  \sum\limits_{j \in S} \alpha_j \bonej \right\rangle - \frac{\tilde{C}_k}{2} \sum\limits_{j \in S} \alpha_j^2.
\end{align}

Maximizing w.r.t. each $\alpha_j$, we get $\alpha_j = \frac{\gradljpbzetaL}{\tilde{C}_k}$ where $\gradljpbzetaL = \max \left(\gradljbzetaL ,0 \right)$. Substituting these values of $\alpha_j$ in (\ref{eq:lowerbound}) gives us the required lower bound, namely
\begin{align}
\label{eq:lowerBfinal}
\lbzetaLS - \lbzetaL \geq \frac{1}{2 \tilde{C}_k} \left\|\nabla l_{S}^+\left(\bzetaL \right) \right\|^2.
\end{align}
\end{proof}

\subsection{Proof of Lemma 5.2}
\begin{proof}
By the definition of $RSC$ constant $c_k$ we find
\begin{align}
& \lbzetaLS -  \lbzetaL \leq \left\langle\nabla \lbzetaL, \bzetaLS-\bzetaL \right\rangle - \frac{c_k}{2} \left\|\bzetaLS-\bzetaL\right\|^2 \nonumber \\
 \label{eq:upperbound}
& \leq \max\limits_{\bv: \bv_{(L\cup S)^c}=0, \bv >=0} \left\langle\nabla \lbzetaL, \bv-\bzetaL \right\rangle - \frac{c_k}{2} \left\|\bv-\bzetaL\right\|^2.
\end{align}
Observe that the KKT conditions at the optimum $\bzetaL$ for the function $f(L)$ necessitates that $\forall j \in L$,
\begin{align*}
\bzetaL_j > 0 &\implies \gradljbzetaL = 0, \\
\bzetaL_j = 0 &\implies \gradljbzetaL \leq 0
\end{align*}
and hence we have $\bv_j = \bzetaL_{j}$. When  $j \in S$, $\bzetaL_j=0$, and maximizing w.r.t. $\bv_j$, the maximum occurs at $\bv_j = \frac{\gradljpbzetaL}{c_k}$ where $\gradljpbzetaL = \max \left(\gradljbzetaL ,0 \right)$. Plugging this maximum value of $\bv$ in (\ref{eq:upperbound}) we get the upper bound
\begin{equation}
\label{eq:upperBfinal}
\lbzetaLS-\lbzetaL  \leq \frac{1}{2 c_k} \left\|\nabla l_{S}^+\left(\bzetaL \right) \right\|^2.
\end{equation}
\end{proof}

\subsection{Proof of Theorem 5.3}
\label{sec:PrDashnonth_Guarantees}
Setting $L=\emptyset$ in Lemma~5.1 we get
\begin{align}
\label{eq:DInequality}
f(S) &\geq \frac{\left\|\nabla l_{S}^+\left(\mathbf{0} \right) \right\|^2}{2 \tilde{C}_m} \geq \frac{\left\|\nabla l_{L^{\ast}}^+\left(\mathbf{0} \right) \right\|^2}{2 \tilde{C}_m} \geq \frac{c_m f(L^{\ast})}{\tilde{C}_m}.
\end{align} 
The second inequality follows from the fact $S$ contains the elements that maximizes the gradient values $\nabla \lbzero$. The third inequality is obtained by setting $S = L^{\ast}$ and $L = \emptyset$ in Lemma~5.2.
Setting $\kappa = \frac{c_m}{\tilde{C}_m}$ we obtain a \emph{constant approximation} of $f\left(S\right) \geq \kappa  f\left(L^{\ast}\right)$.

\subsection{Proof of Lemma 5.4}
Recall that given two disjoint sets $L$ and $S$, the submodularity ratio is defined as
\begin{equation*}
\gamma_{L,S} = \frac{\sum\limits_{j \in S} \left[f(L \cup \{j\}) - f(L)\right]} {f(L \cup S) - f(L)}.
\end{equation*}
where $f(L) = \lbzetaL$ and $f(L \cup S) = \lbzetaLS$.  Using inequalities (\ref{eq:lowerBfinal}) and (\ref{eq:lowerBfinal}) we can bound the submodularity ratio as
\begin{equation}
\label{eq:submodbounds}
\frac{c_{|L|+|S|}}{\tilde{C}_1} \leq \gamma_{L,S} \leq \frac{\tilde{C}_{|S|}}{c_{|L|+1}}. 
\end{equation}

\subsection{Proof of Theorem 5.5}
\label{sec:PrGreedyGuarantees}
As the set $S$ consists of those $m$ elements where the function evaluation on the singleton sets is the maximal, we have $\forall j \in S$; $ f(\{j\}) \geq f(\{p\})$; $p \notin  S$. When compared with the optimal set $L^{\ast}$ we find
\begin{align*}
f(S) = \frac{\sum\limits_{j \in S} f(\{j\})}{\gamma_{\emptyset,S}}  &\geq \frac{1}{R_m} \left[\sum\limits_{p \in L^{\ast}} f\left(\{p\}\right)\right] =\frac{\gamma_{\emptyset,L^{\ast}}}{R_m} f\left(L^{\ast}\right) \geq \frac{r_m}{R_m} f\left(L^{\ast}\right).
\end{align*}
Thus $f(S) \geq \kappa f\left(L^{\ast}\right)$ where $\kappa = \frac{r_m}{R_m}$.

\subsection{Proof of Lemma 5.6}
Recall that an incoming element $j$ is added to the set $L_{\tau}$ provided 
\begin{equation}
\label{eq:thcriteria}
\gradljbzetaLtau \geq \sqrt{\frac{2\tau}{m}}.
\end{equation}
By setting $S$ to be singleton set $\{j\}$ in Lemma~5.1 we get 
\begin{equation*}
f\left(L_{\tau} \cup \{j\}\right) - f\left(L_{\tau}\right) \geq \frac{1}{2 \tilde{C}_1} \left[\gradljpbzetaLtau \right]^2 \geq \frac{\tau}{\tilde{C}_1 m}.
\end{equation*}
So by adding $\{j\}$ to the current set $L_{\tau}$, the increase the set function is at least $\frac{\tau}{\tilde{C}_1 m}$. When $\left|L_{\tau}\right| = m$, it follows that $f\left(L_{\tau}\right) \geq \frac{\tau}{\tilde{C}_1}$.

\section{Additional Experiments}
\label{sec:additionalexp}
Here we report additional experiments that further underscores the usefulness of our algorithm .
\subsection{Letters: Fonts and Stroke Styles}
\label{sec:fontstrokeexp}
As mentioned in Section 6.2, we know that the letters dataset spans 20 different fonts and 5 different stroke styles. It has been known from previous studies \cite{Fern04, letters} that one could cluster any letter into 20 groups and partition based on the fonts. Analogously clustering into 5 groups can largely uncover the different stroke styles. 

Given this we wanted to see if our prototypes from ProtoStream span the different fonts and styles. Since the partitions are not given we perform k-means clustering and partition copies of each letter into 20 and then 5 groups. We assigned each of our 500 prototypes to the closest cluster based on euclidean distance. We then plotted a histogram of what fraction of instances belonged to which cluster. We also compared this with assignment to randomly formed clusters so as to verify that the clustering in fact had some information.

These results are seen in Figures~\ref{font} and \ref{style}. The more uniform the distribution the better. We see clearly that our prototypes are quite equitably distributed across the different clusters with being much superior than random. This implies two things. First, that the clusters do capture information of possibly fonts and styles. Secondly, our prototypes nicely span these fonts and styles again verifying that ProtoStream selects diverse informative instances.

\begin{figure*}[t]
  \begin{center}
    \begin{tabular}{cc}
      \includegraphics[width=0.49\linewidth]{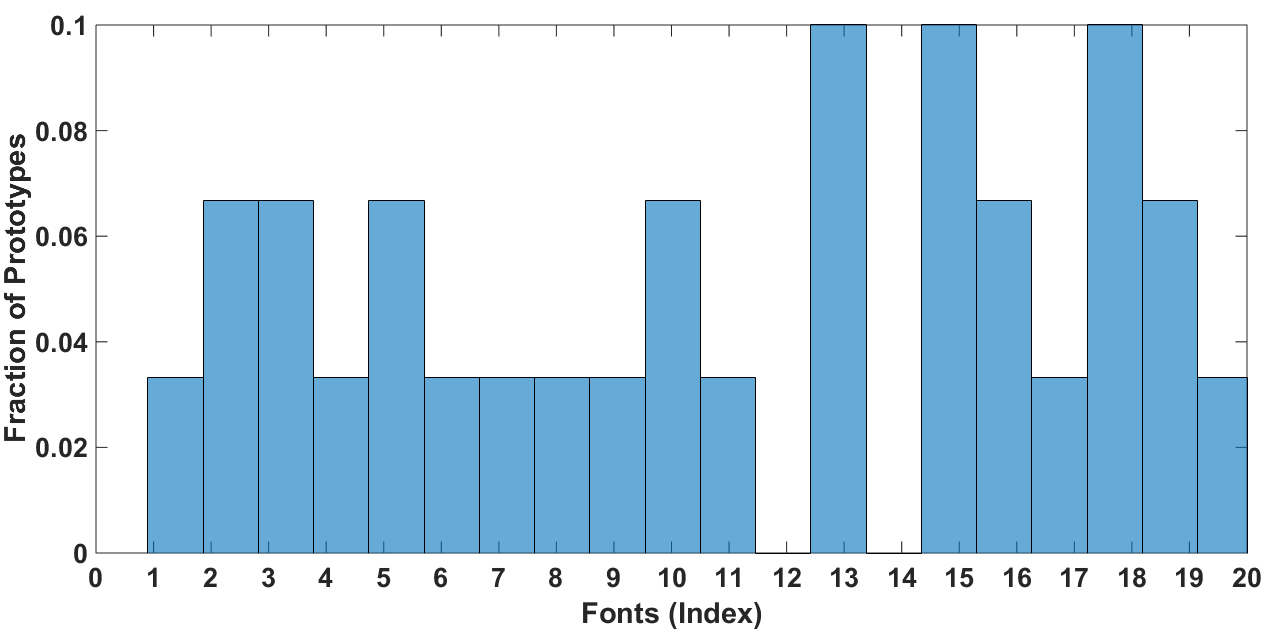} &      
      \includegraphics[width=0.49\linewidth]{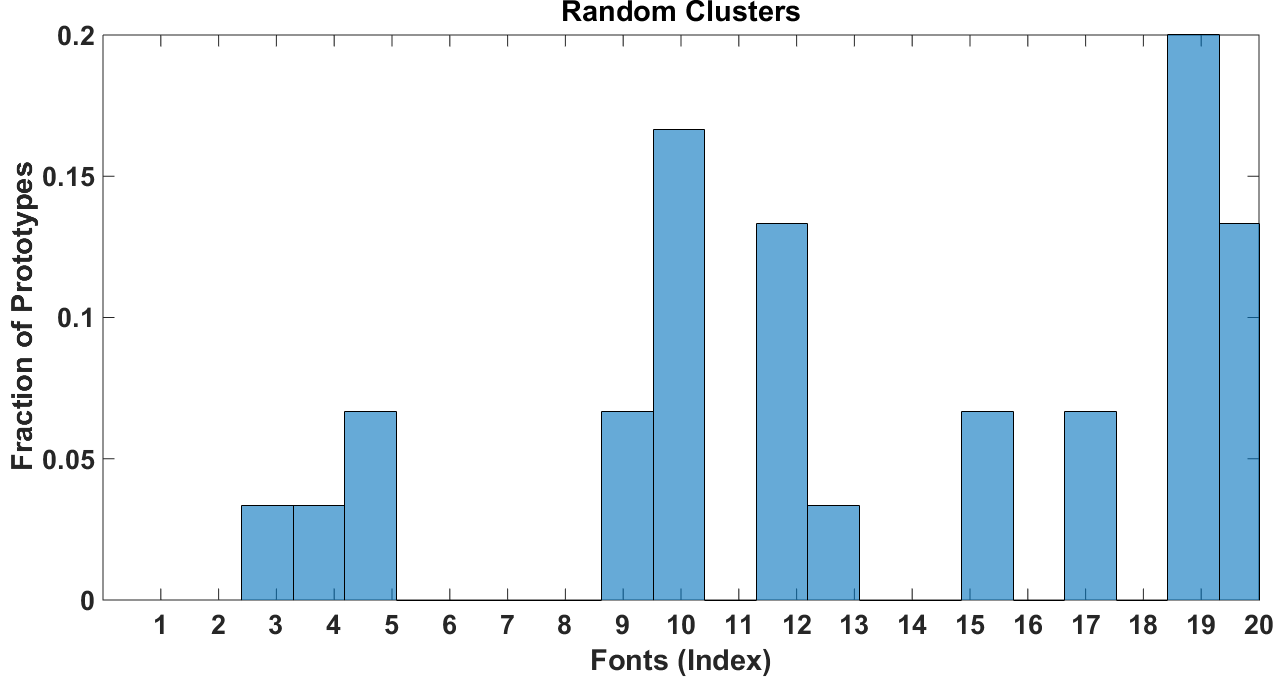}
     \end{tabular}
  \end{center}
  \caption{Above (left) we see the distribution of our selected prototypes across 20 clusters each associated with a different font. The right figure depicts the distribution when we form random clusters.}
  \label{font}
\end{figure*}

\begin{figure*}[t]
  \begin{center}
    \begin{tabular}{cc}
      \includegraphics[width=0.49\linewidth]{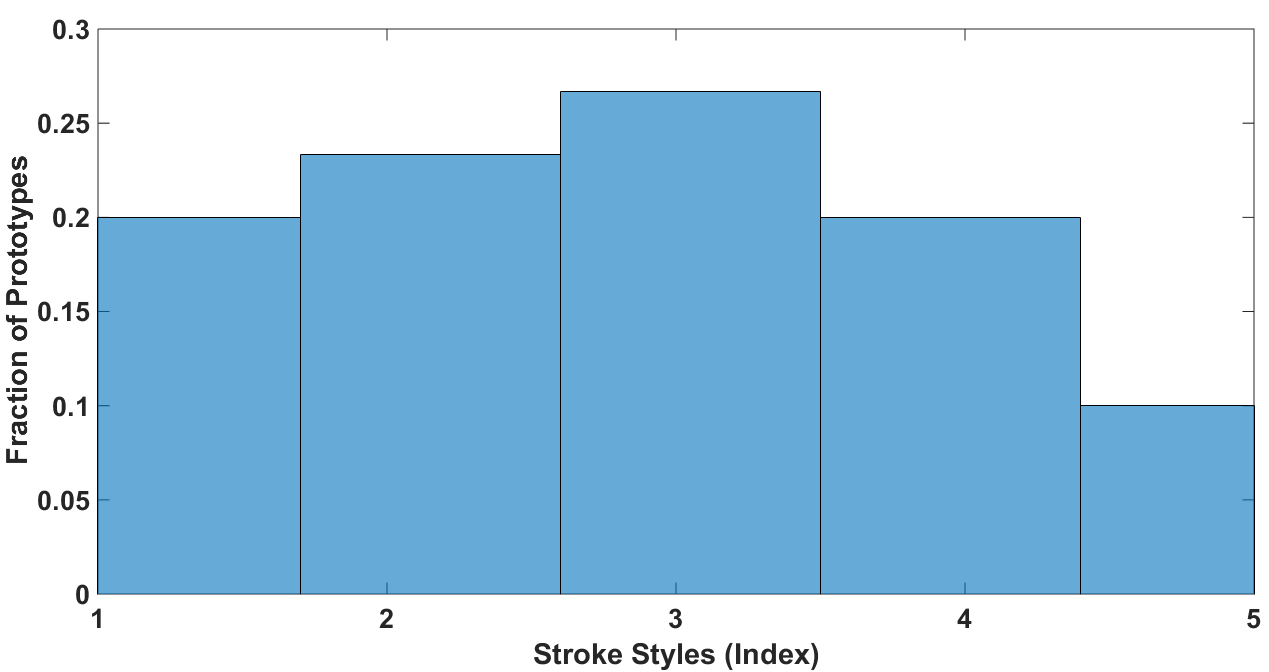} &      
      \includegraphics[width=0.49\linewidth]{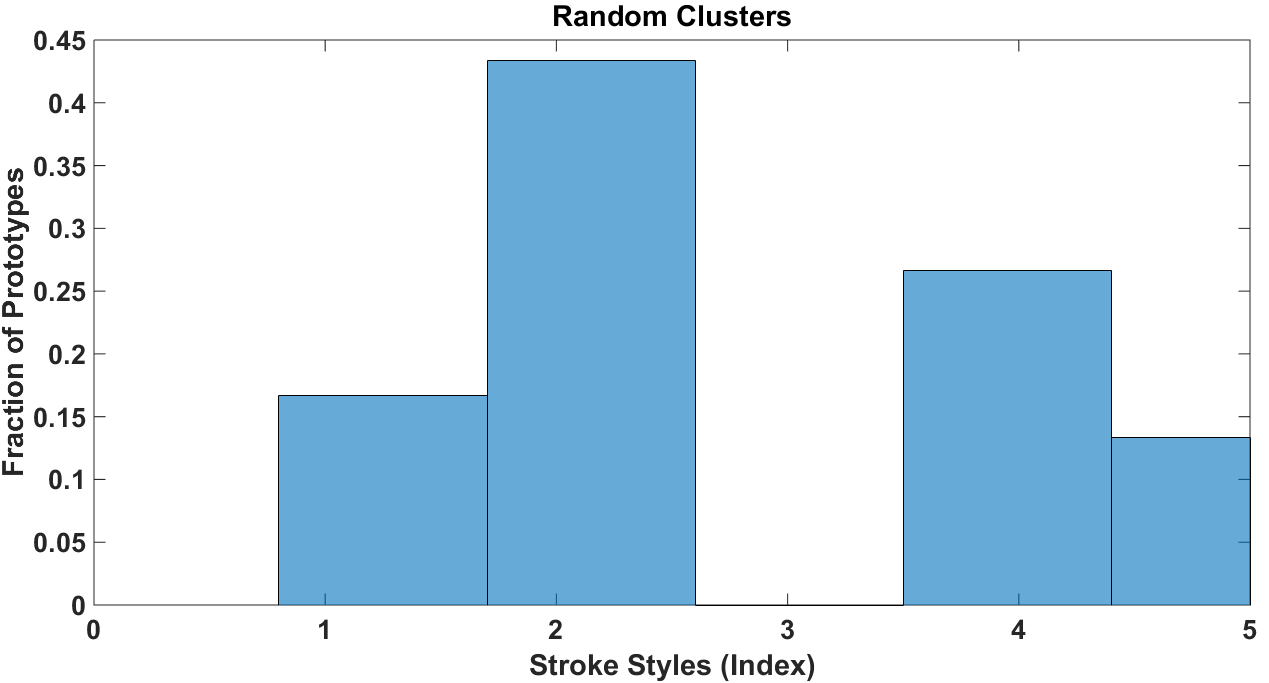}
     \end{tabular}
  \end{center}
  \caption{Above (left) we see the distribution of our selected prototypes across 5 clusters each associated with a different stroke style. The right figure depicts the distribution when we form random clusters.}
  \label{style}
\end{figure*}
\subsection{Adapting to Target Dataset}
\label{sec:skeweddistributionexp}
The plots in Figures~\ref{mnist} and \ref{letters} indirectly appraise the quality of the selected prototypes based on their accuracy in classifying a test set. In this section we design experiments from which we can directly infer the prototype selection quality by studying how well our algorithms adapt to a different test or target distribution. To this end, we create target datasets having samples only from a single class (digits in MNIST and letters in UCI). For example, we create a target dataset for the digit 1 by selecting only $1's$ from the original test of 60000. Given the original source dataset $X^{(2)}$ which contain (almost) an equal mix of different digits or letters, the goal is to see how well our algorithms adapt to these heavily skewed target distributions $X^{(1)}$ that contain only a single digit/alphabet. In other words, we wish to evaluate whether they still just pick a uniform distribution over all the digits/letters from  $X^{(2)}$ or adapt and pick more prototypes of the target digit. Selecting prototypes from one source set that matches well with a different target distribution are natural in covariate shift correction settings \cite{Agarwal11}. 

For going across datasets, we optimize the cost function:	
\begin{equation}
\label{eq:lacross} 
\mbox{Maximize   }l\left(\w\right) =  \w^T \bmu - \frac{1}{2} \w^T K \w \mbox{   s.t.  } \|\w\|_0 \leq m \mbox{ and } \w \geq 0
\end{equation}
where as before $K$ is the positive definite Kernel matrix with entries $K_{i,j} = k(\x_i,\x_j), \forall \x_i, \x_j \in X^{(2)}$ and the entries of the vector $\boldsymbol \mu$ contains the mean inner product of a data sample in $X^{(2)}$ with the target $X^{(1)}$ and is given by: $\mu_{j} = \frac{1}{n^{(1)}} \sum\limits_{i=1}^{n^{(1)}} k(\y_i,\x_j); \forall \x_j \in X^{(2)}$. Here $n^{(1)} = \left|X^{(1)}\right|$. Note that the labels of the target samples are not exposed to the algorithms.  The prototype selection quality can be quantified from the percentage of selected prototypes that match target class. Higher the percentage, better is the selection quality.

We see in Figure~\ref{MNISTskewd} that our algorithms along with Streak do adapt to the target distribution. In fact, ProtoBasic almost exclusively picks examples of the target digit in MNIST showcasing its effectiveness in such a setting. The relative running times are similar to those reported in the main document. Given this, ProtoBasic could be the most preferred method in scenarios where the target dataset more or less contains a single class.

\begin{figure*}[t]
  \begin{center}
    \begin{tabular}{ccc}
      \includegraphics[width=0.33\linewidth]{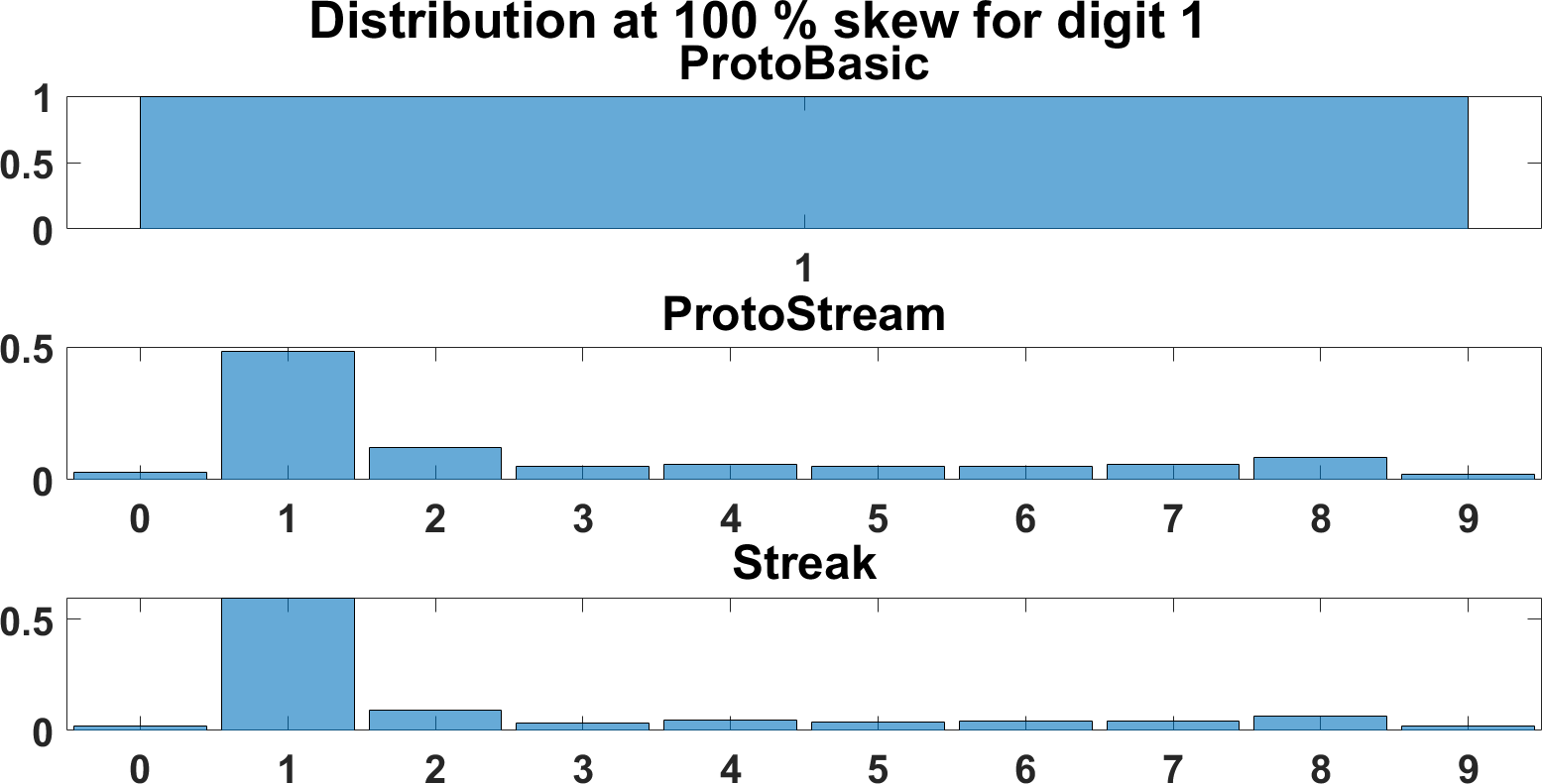} &      
      \includegraphics[width=0.33\linewidth]{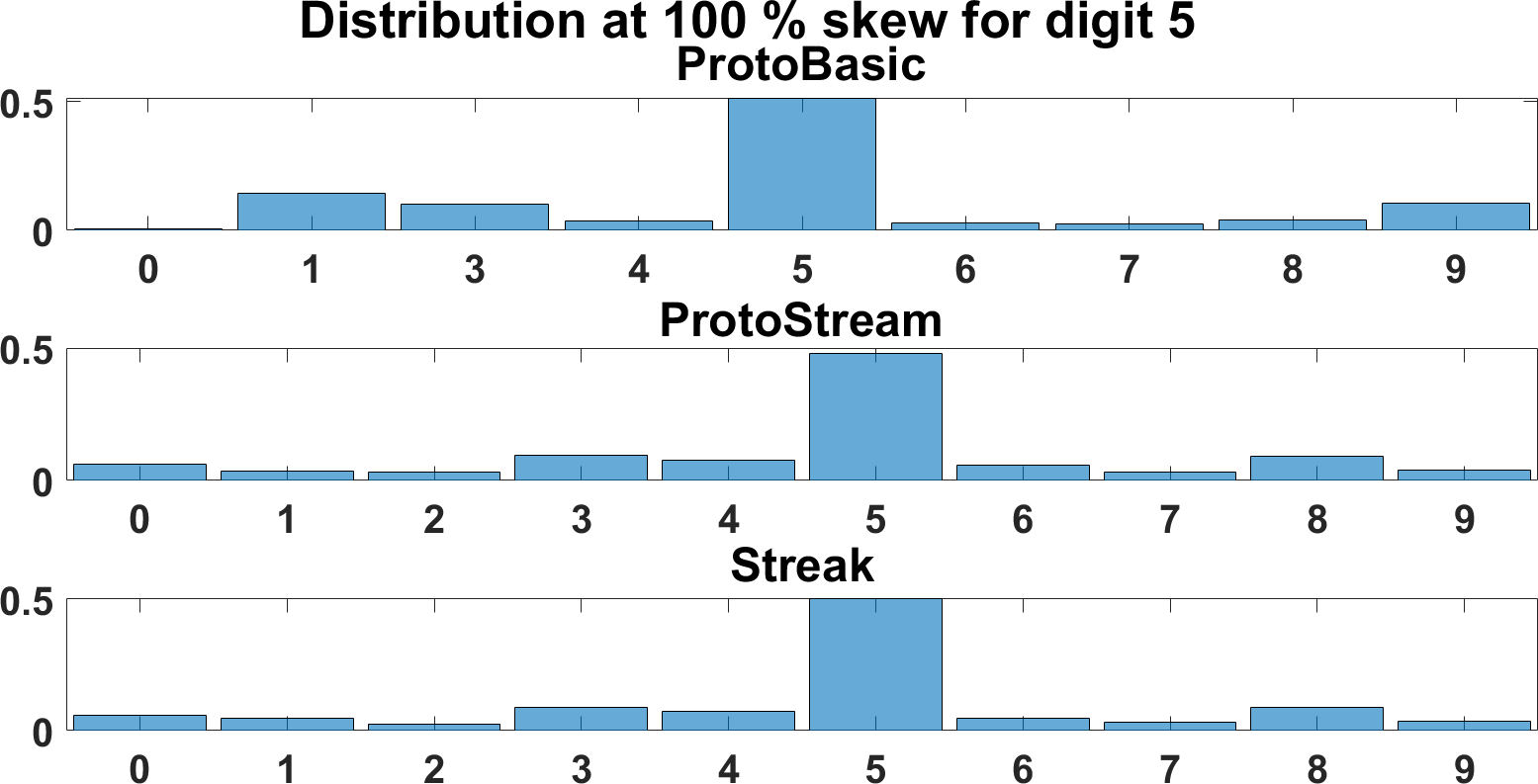} &
		\includegraphics[width=0.33\linewidth]{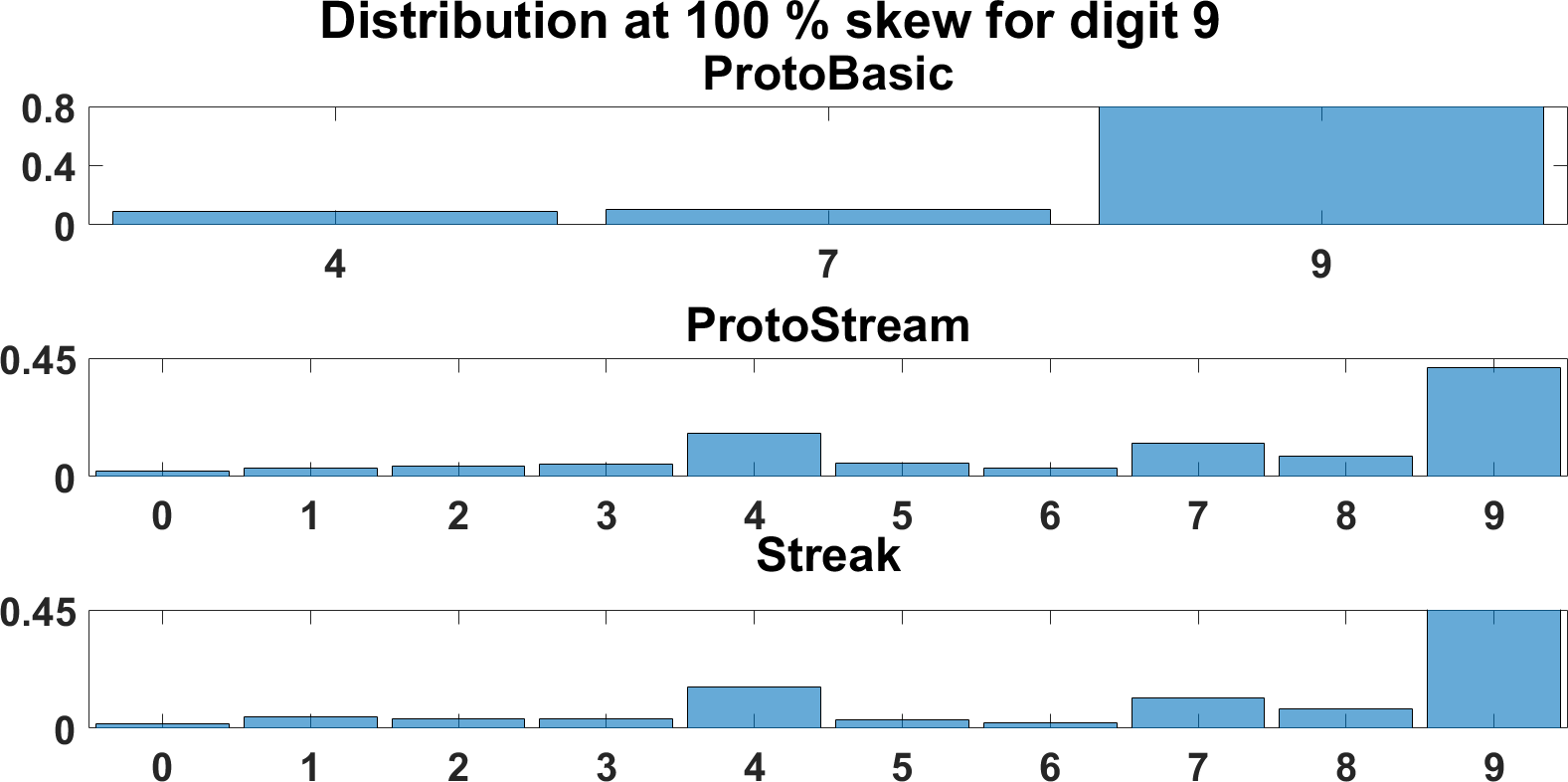} \\
 		\includegraphics[width=0.33\linewidth]{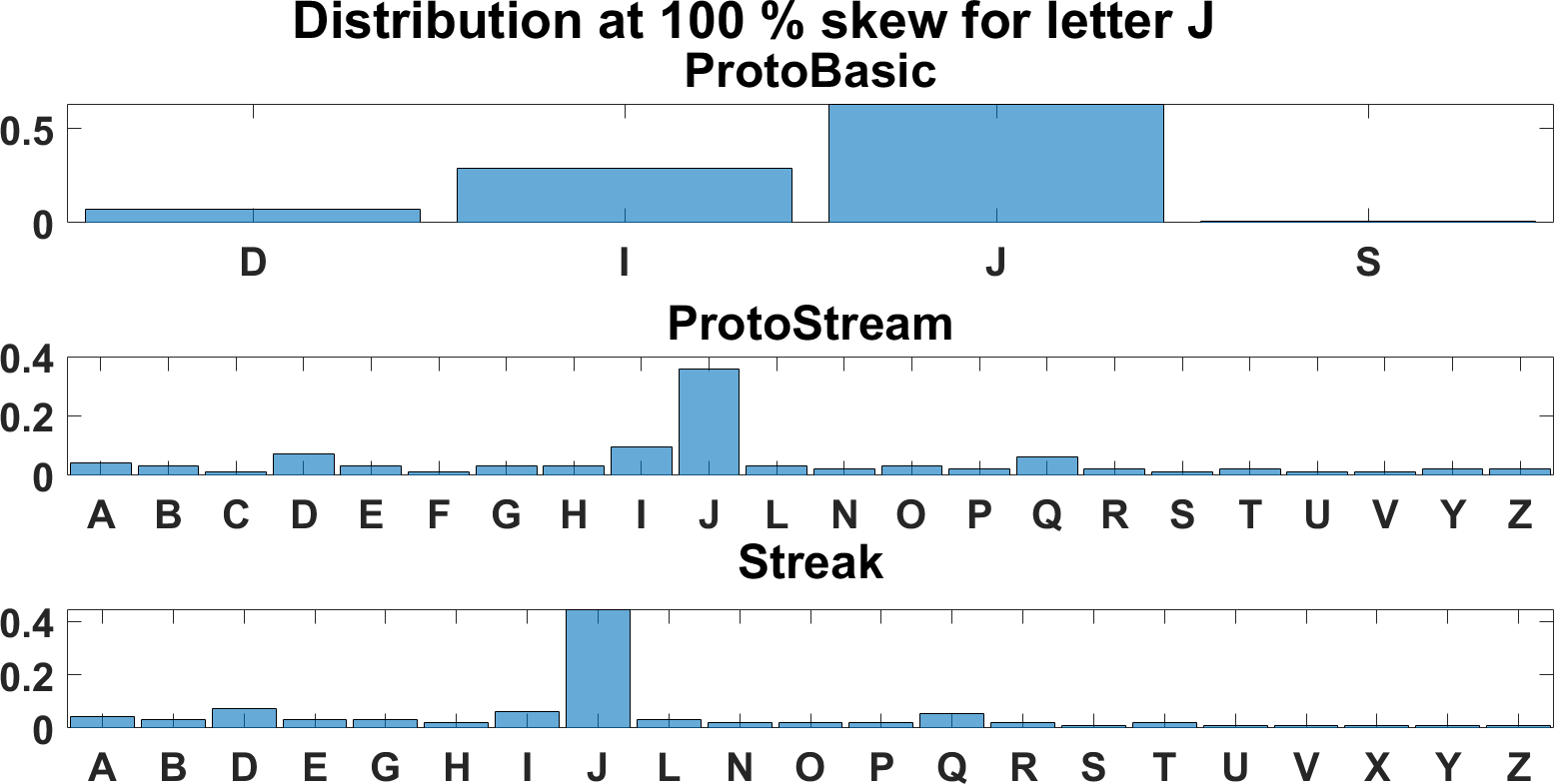} &      
      \includegraphics[width=0.33\linewidth]{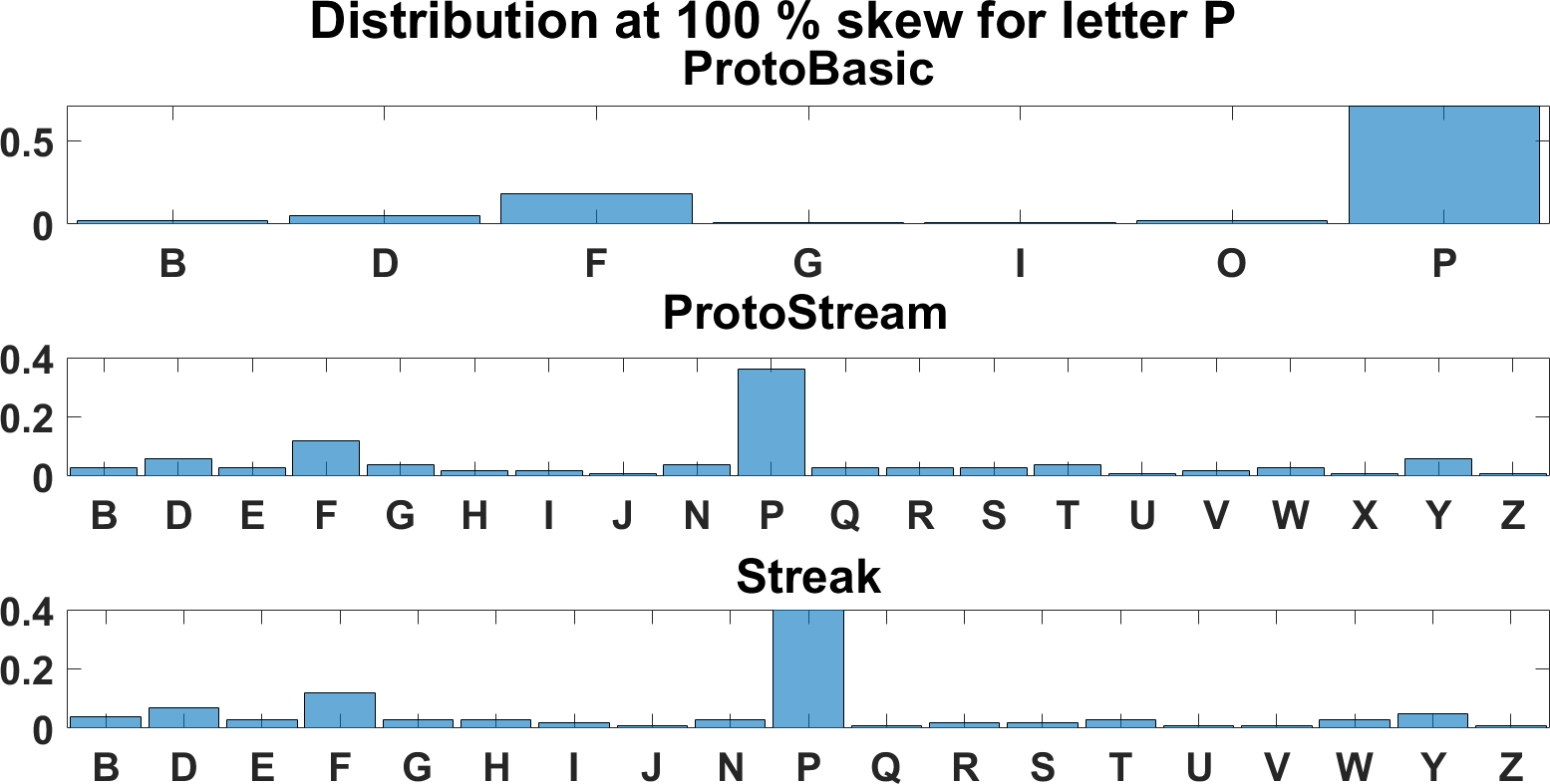} &
		\includegraphics[width=0.33\linewidth]{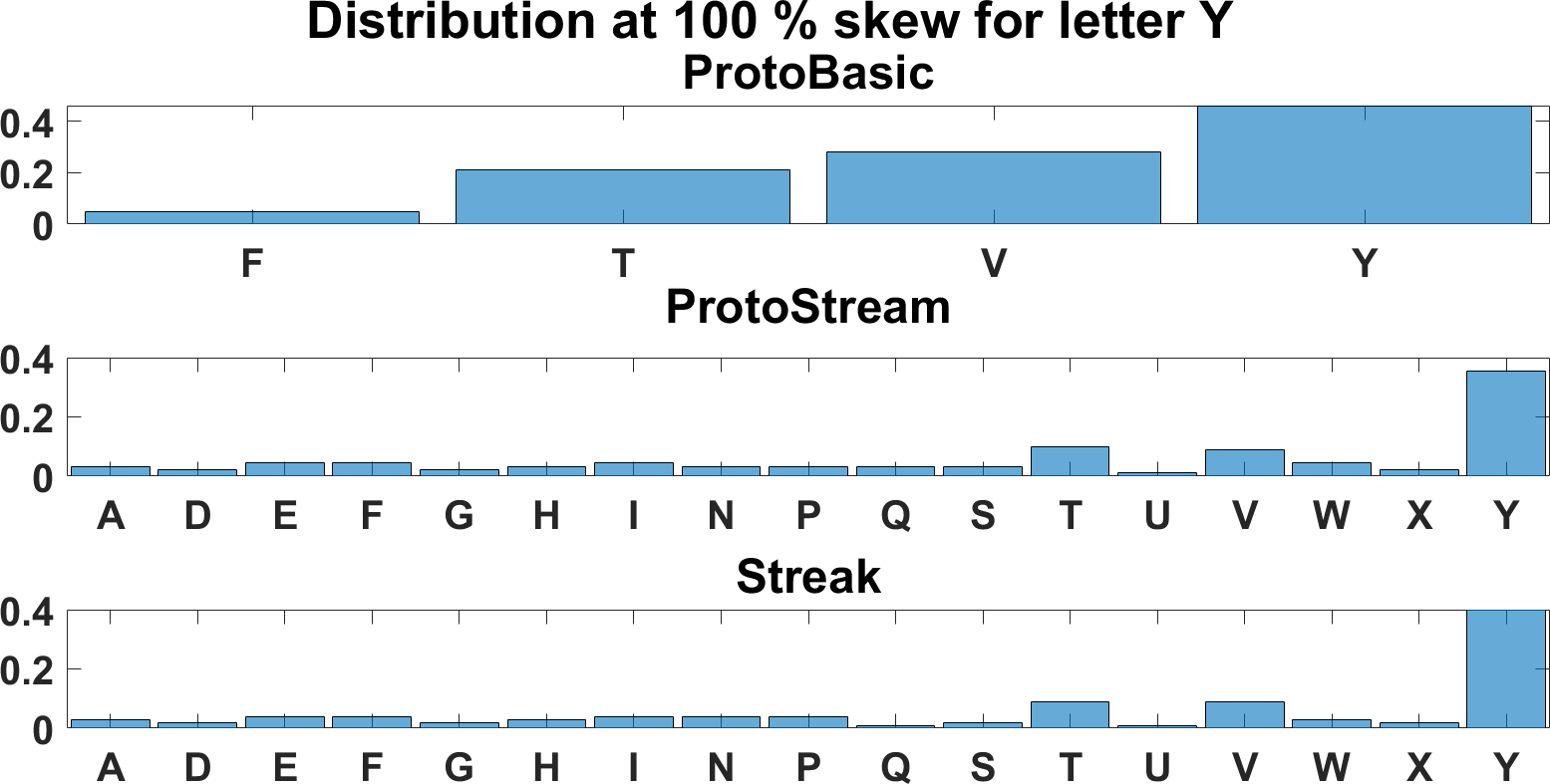} \\
     \end{tabular}
  \end{center}
  \caption{Top: Distribution of MNIST digits picked up from the source dataset when the target dataset contains images of just the digit 1 (left), 5 (center) and 9 (right). Bottom: Distribution of UCI letters picked up from the source dataset when the target dataset contains images of just the letters J (left), P (center) and Y (right). In this case ProtoBasic, is competitive (if not better) with other methods as most of the chosen prototypes are from the target class.}
  \label{MNISTskewd}
\end{figure*}

\bibliography{protostream}

\begin{thebibliography}{10}

\bibitem{Agarwal11}
D.~Agarwal, L.~Li, and A.~J. Smola.
\newblock {Linear-Time Estimators for Propensity Scores}.
\newblock In {\em {$14^{th}$ Intl. Conference on Artificial Intelligence and
  Statistics (AISTATS)}}, pages 93--100, 2011.

\bibitem{substream}
A.~Badanidiyuru, B.~Mirzasoleiman, A.~Karbasi, and A.~Krause.
\newblock Streaming submodular maximization: Massive data summarization on the
  fly.
\newblock In {\em Proceedings of the 20th ACM SIGKDD International Conference
  on Knowledge Discovery and Data Mining}, pages 671--680. ACM, 2014.

\bibitem{sproto}
J.~Bien and R.~Tibshirani.
\newblock {Prototype Selection for Interpretable Classification}.
\newblock {\em Ann. Appl. Stat.}, pages 2403--2424, 2011.

\bibitem{weaksubInit}
A.~Das and D.~Kempe.
\newblock {Submodular meets Spectral: Greedy Algorithms for Subset Selection,
  Sparse Approximation and Dictionary Selection}.
\newblock In {\em Intl. Conference on Machine Learning (ICML)}, 2011.

\bibitem{weaksub}
E.~Elenberg, R.~Khanna, A.~G. Dimakis, and S.~Negahban.
\newblock {Restricted Strong Convexity Implies Weak Submodularity}.
\newblock In {\em https://arxiv.org/abs/1612.00804}, 2017.

\bibitem{weakstream}
E.~R. Elenberg, A.~G. Dimakis, M.~Feldman, and A.~Karbasi.
\newblock Streaming weak submodularity: Interpreting neural networks on the
  fly.
\newblock {\em Advances in Neural Inf. Processing}, 2017.

\bibitem{Fern04}
X.~Z. Fern and C.~Brodley.
\newblock {Cluster Ensembles for High Dimensional Clustering: An Empirical
  Study}.
\newblock {\em Machine Learning Research}, 22, January 2004.

\bibitem{automl}
M.~Feurer, K.~E. Aaron~Klein, J.~Springenberg, M.~Blum, and F.~Hutter.
\newblock Efficient and robust automated machine learning.
\newblock Advances in Neural Information Processing Systems Workshop, 12 2015.

\bibitem{letters}
P.~W. Frey and D.~J. Slate.
\newblock Letter recognition using holland-style adaptive classifiers.
\newblock {\em Machine Learning}, 6(2), 1991.

\bibitem{fujishige05}
S.~Fujishige.
\newblock {\em Submodular functions and optimization}.
\newblock Number~58 in Annals of Discrete Mathematics. Elsevier Science, 2
  edition, 2005.

\bibitem{proto}
K.~Gurumoorthy, A.~Dhurandhar, and G.~Cecchi.
\newblock Protodash: Fast interpretable prototype selection.
\newblock In {\em https://arxiv.org/abs/1707.01212v2}, 2017.

\bibitem{Kim16}
B.~Kim, R.~Khanna, and O.~Koyejo.
\newblock {Examples are not Enough, Learn to Criticize! Criticism for
  Interpretability}.
\newblock In {\em {$30^{th}$ Conference on Neural Information Processing
  Systems (NIPS)}}, 2016.

\bibitem{mnist}
Y.~LeCun, L.~Bottou, Y.~Bengio, and P.~Haffner.
\newblock Gradient-based learning applied to document recognition.
\newblock In {\em Proceedings of the IEEE}, pages 2278--2324, 1998.

\bibitem{nmf}
D.~D. Lee and H.~S. Seung.
\newblock Algorithms for non-negative matrix factorization.
\newblock In {\em In NIPS}, pages 556--562. MIT Press, 2001.

\bibitem{lo83}
L.~Lov{\'{a}}sz.
\newblock {\em Mathematical programming -- The State of the Art}, chapter
  Submodular Functions and Convexity, pages 235--257.
\newblock Springer, 1983.

\bibitem{rsceg}
S.~Negahban, B.~Yu, M.~J. Wainwright, and P.~K. Ravikumar.
\newblock A unified framework for high-dimensional analysis of m-estimators
  with decomposable regularizers.
\newblock In Y.~Bengio, D.~Schuurmans, J.~D. Lafferty, C.~K.~I. Williams, and
  A.~Culotta, editors, {\em Advances in Neural Information Processing Systems
  22}, pages 1348--1356. 2009.

\bibitem{Nemhauser78}
G.~L. Nemhauser, L.~A. Wolsey, and M.~L. Fisher.
\newblock {An Analysis of Approximations for Maximizing Submodular Set
  Functions}.
\newblock {\em Math. Program.}, 14:265--294, December 1978.

\end{thebibliography}
\bibliographystyle{abbrv}
\end{document}